\documentclass[twoside]{article}

\usepackage[accepted]{aistats2022}

\usepackage[round]{natbib}

\usepackage[font=small]{caption}

\usepackage{graphicx}

\newcommand{\hiddensection}[1]{
    \addtocontents{toc}{\protect\setcounter{tocdepth}{0}}
    \section{#1}
    \addtocontents{toc}{\protect\setcounter{tocdepth}{6}}
}

\usepackage{dsfont}
\usepackage{setspace}
\usepackage[frozencache]{minted}
\usepackage{cancel}
\usepackage{booktabs}
\usepackage{mathtools}
\usepackage{units}
\usepackage{amsmath}
\usepackage{amssymb}
\usepackage{amsthm}
\usepackage{upgreek}
\usepackage{xcolor}
\usepackage{hyperref}
\usepackage{bm}
\usepackage{bbm}
\usepackage{enumitem}

\renewcommand{\paragraph}[1]{{\textbf{#1}}}

\definecolor{darkblue} {rgb} {0.0 , 0.0 , 0.65}
\definecolor{darkred}  {rgb} {0.80, 0.0 , 0.0 }
\definecolor{darkgreen}{rgb} {0.0 , 0.50, 0.0 }
\definecolor{gray75}   {gray}{0.75}

\usepackage[
    noabbrev,
    capitalize,
    nameinlink]{cleveref}    
    
\newlength{\thmtopsep}
\newlength{\thmbotsep}
\newtheoremstyle{theoremstyle}
    {\thmtopsep}{\thmbotsep}
    {}           
    {}           
    {\bfseries}  
    {.}          
    {.5em}       
    {}           
\theoremstyle{theoremstyle}

\newtheorem{definition}{Definition}[section]

\newtheorem{lemma}{Lemma}[section]
\newtheorem{model}{Model}[section]
\newtheorem{proposition}{Proposition}[section]

\crefname{corollary}{Cor}{Cors}
\Crefname{corollary}{Cor}{Cors}
\crefname{definition}{Definition}{Definitions}
\Crefname{definition}{Definition}{Definitions}
\crefname{example}{Example}{Examples}
\Crefname{example}{Example}{Examples}
\crefname{fact}{Fact}{Facts}
\Crefname{fact}{Fact}{Facts}
\crefname{lemma}{Lem}{Lems}
\Crefname{lemma}{Lem}{Lems}
\crefname{model}{Mod}{Mods}
\Crefname{model}{Mod}{Mods}
\crefname{proposition}{Prop}{Props}
\Crefname{proposition}{Prop}{Props}
\crefname{remark}{Rem}{Rems}
\Crefname{remark}{Rem}{Rems}
\crefname{section}{Sec}{Secs}
\Crefname{section}{Sec}{Secs}
\crefname{theorem}{Thm}{Thms}
\Crefname{theorem}{Thm}{Thms}
\crefname{equation}{}{}
\Crefname{equation}{}{}
\crefname{appendix}{App}{Apps}
\Crefname{appendix}{App}{Apps}
\crefname{figure}{Fig}{Figs}
\Crefname{figure}{Fig}{Figs}

\newcommand{\ie}{\textit{i.e.}}
\newcommand{\eg}{\textit{e.g.}}

\newcommand{\R}{\mathbb{R}}

\newcommand{\Var}{\mathbb{V}}
\newcommand{\E}{\mathbb{E}}
\newcommand{\N}{\mathbb{N}}
\renewcommand{\H}{\mathbb{H}}

\newcommand{\F}{\mathcal{F}}

\newcommand{\sd}{\mathrm{d}}
\renewcommand{\d}{\partial}
\newcommand{\isd}{\,\mathrm{d}}
\newcommand{\KL}{\operatorname{KL}}
\newcommand{\GP}{\mathcal{GP}}
\newcommand{\Normal}{\mathcal{N}}
\newcommand{\vardot}{\,\cdot\,}
\newcommand{\cond}{\,|\,}
\newcommand{\divsep}{\,\|\,}
\newcommand{\T}{\top}

\renewcommand{\th}{\theta}
\renewcommand{\ss}[1]{_{\text{#1}}}
\newcommand{\erf}{\operatorname{erf}}
\newcommand{\sub}{\subseteq}
\renewcommand{\P}{\mathbb{P}}
\newcommand{\tr}{\operatorname{tr}}

\newcommand{\ind}{\mathds{1}}

\newcommand{\ep}{\varepsilon}

\newcommand{\vu}{\mathbf{u}}
\newcommand{\vt}{\mathbf{t}}

\newcommand{\vy}{\mathbf{y}}
\newcommand{\vz}{\mathbf{z}}
\newcommand{\vx}{\mathbf{x}}

\newcommand{\mK}{\mathbf{K}}
\newcommand{\mM}{\mathbf{M}}
\newcommand{\mI}{\mathbf{I}}

\newcommand{\mA}{\mathbf{A}}
\newcommand{\mB}{\mathbf{B}}

\newcommand{\mX}{\mathbf{X}}
\newcommand{\mY}{\mathbf{Y}}

\newcommand{\mSigma}{\mathbf{\Sigma}}
\newcommand{\vmu}{\bm{\upmu}}
\newcommand{\vbeta}{\bm{\upbeta}}

\DeclarePairedDelimiter\parens{(}{)}             
\DeclarePairedDelimiter\sbrac{[}{]}              
\DeclarePairedDelimiter\set{\{}{\}}              
\DeclarePairedDelimiter\lra{\langle}{\rangle}    
\DeclarePairedDelimiter\floor{\lfloor}{\rfloor}  
\DeclarePairedDelimiter\norm{\|}{\|}             
\DeclarePairedDelimiter\abs{|}{|}                

\begin{document}

\twocolumn[

\aistatstitle{Modelling Non-Smooth Signals with Complex Spectral Structure}

\aistatsauthor{
    Wessel P.\ Bruinsma \And Martin Tegn\'er* \And Richard E.\ Turner
}
\runningauthor{
    Wessel P.\ Bruinsma, Martin Tegn\'er, and Richard E.\ Turner
}

\aistatsaddress{
    University of Cambridge \\
    Invenia Labs \\
    \texttt{wpb23@cam.ac.uk}
    \And
    University of Oxford \\
    Oxford-Man Institute \\
    \texttt{mt@robots.ox.ac.uk}
    \And
    University of Cambridge \\
    ~\\
    \texttt{ret26@cam.ac.uk}
}

]

\begin{abstract}
    The Gaussian Process Convolution Model \citep[GPCM;][]{Tobar:2015:Learning_Stationary} is a model for signals with complex spectral structure.
    A significant limitation of the GPCM is that it assumes a rapidly decaying spectrum: it can only model smooth signals.
    Moreover, inference in the GPCM currently requires (1) a mean-field assumption, resulting in poorly calibrated uncertainties, and (2) a tedious variational optimisation of large covariance matrices.
    We redesign the GPCM model to induce a richer distribution over the spectrum with relaxed assumptions about smoothness: the Causal Gaussian Process Convolution Model (CGPCM) introduces a causality assumption into the GPCM, and the Rough Gaussian Process  Convolution Model (RGPCM) can be interpreted as a Bayesian nonparametric generalisation of the fractional Ornstein--Uhlenbeck process.
    We also propose a more effective variational inference scheme, going beyond the mean-field assumption: we design a Gibbs sampler which directly samples from the optimal variational solution, circumventing any variational optimisation entirely.
    The proposed variations of the GPCM are validated in experiments on synthetic and real-world data, showing promising results.
    \vspace{-0.5em}
\end{abstract}

\hiddensection{INTRODUCTION}

Gaussian processes (GPs) form a popular and powerful probabilistic framework for modelling functions \citep{Rasmussen:2006:Gaussian_Processes}.
They are successfully applied in a wide variety of contexts and are state of the art in numerous regression tasks \citep{Bui:2016:Deep_Gaussian_Processes_for_Regression}.
Gaussian processes are nonparametric models that grow in complexity as more data is observed, which makes them robust against overfitting.
They achieve this automatic calibration of complexity by posing a prior distribution directly over the underlying function $f\colon \mathcal{T} \to \R$.
In particular, the defining property of a Gaussian process is
that any finite collection of function values $f(t_1), \ldots, f(t_n)$ is multivariate Gaussian distributed.

The key modelling decision when using Gaussian processes is the choice of covariance function $k(t, t') = \operatorname{cov}(f(t), f(t'))$, also called the \emph{kernel}.
The kernel encodes prior information about the underlying function $f$.
For example, the kernel specifies the smoothness of $f$ and the typical length scale on which $f$ varies.
A kernel is \emph{stationary} if it only depends on the difference of its arguments: $k(t, t') = k(t - t')$.
In that case, if the data is translated, the predictions are translated accordingly, a symmetry called \emph{translation equivariance} which is often desirable.
A stationary kernel is characterised by its Fourier transform---a fact known as Bochner's theorem---where the Fourier transform is called the  \emph{power spectral density} (PSD) or simply spectrum.
If $f$ is decomposed into complex exponentials with random amplitudes, then the spectrum tells us how the variances of these random amplitudes vary with frequency.

A popular choice for the kernel is the \emph{exponentiated quadratic} (EQ) kernel:
$
    k(t, t') = \exp\parens*{-\frac{1}{2\ell^2}\|t - t'\|^2}.
$
The EQ kernel assumes that $f$ is infinitely differentiable and varies on only a single length scale $\ell$.
Whilst appropriate for many tasks, these assumptions are too rigid for harder regression problems, which require more expressive kernels.
From the perspective of the spectrum, the EQ kernel assumes that the spectrum of $f$ is necessarily of the form $\operatorname{PSD}(\omega) = c_1 e^{-c_2 \omega^2}$.
This is restrictive, because real-world signals often have much richer spectral structure.

An important property of the spectrum is the behaviour at high frequencies. Specifically, the asymptotic decay of $\operatorname{PSD}(\omega)$ is intimately connected to the regularity of sample paths of $f$.
For example, sample paths of $f$ are $n$-times differentiable if and only if the $2n$\textsuperscript{th} spectral moment is finite, $\int \omega^{2n}\text{PSD}(\omega)\isd \omega<\infty$ \citep[Theorem 4,][]{Cambanis:1973:On_Some_Continuity_and_Differentiability}.
Since the PSD of the EQ kernel $\text{PSD}(\omega) = c_1 \smash{e^{-c_2 \omega^2}}$ decays quicker than any polynomial, sample paths of a Gaussian process with an EQ kernel are infinitely differentiable.

Developing flexible and expressive kernels as well as choosing the right kernel for a particular task is an active area of research.
Many approaches let the kernel be of a flexible parametric form \citep{Wilson:2013:Spectral_Mixture,Calandra:2016:Manifold_Gaussian_Processes_for_Regression,Sun:2018:Differentiable_Compositional_Kernel_Learning_for} or search over a large space of kernels \citep{,Grosse:2012:Exploiting_Compositionality_to_Explore,Duvenaud:2014:Automatic_Construction,Malkomes:2016:Bayesian_Optimization_for_Automated_Model}.
These approaches are often effective, but run the risk of overfitting due to the large number of parameters they introduce.
Morever, posing a flexible parametric form for the kernel or spectrum typically presents an optimisation problem which is riddled with local optima.
\begin{figure}[t]
    \centering
    \vspace{.5em}
    \includegraphics[trim={19cm 0 0 0}, clip, width=\linewidth]{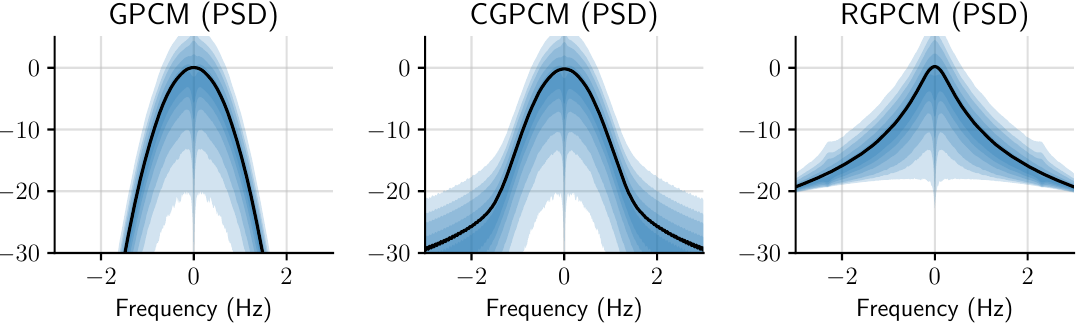}
    \vspace*{-1.6em}
    \caption{
        Visualisation of the nonparametric prior over the power spectral density by the GPCM (left) and the two variants of the GPCM introduced in this paper (middle and right).
        Observe that the GPCM has a quickly decaying spectrum, whereas the CGPCM and especially the RGPCM have support at higher frequencies.
    }
    \vspace{-1.25em}
    \label{fig:priors_intro}
\end{figure}
Other approaches treat the kernel as a latent function by assuming a prior distribution over the kernel  \citep{Tobar:2015:Learning_Stationary,Oliva:2016:Bayesian_Nonparametric_Kernel-Learning,Jang:2017:Scalable_Levy_Process_Priors_for}. 
This induces a prior distribution over the PSD---see Figure \ref{fig:priors_intro} for an illustration---which is appealing, 
because it brings the benefits of Bayesian nonparametrics to the spectrum:
as more data are observed, more spectral structure of the data is revealed, and the posterior over the spectrum automatically increases in complexity.
Inference in these models, however, is considerably more involved and computationally demanding.

An example of a model that treats the kernel as a random function is the Gaussian Process Convolution Model (GPCM) \citep{Tobar:2015:Learning_Stationary}, which is the focus of this paper.
The GPCM parametrises the kernel of the Gaussian process $f$ with a
sample of \emph{another} Gaussian process $h$,
where in turn $h$ has a modified EQ kernel.
\citet{Tobar:2015:Inter-Domain_Inducing} construct the GPCM by considering a linear system excited by white noise;
this construction will be central in this paper.

Although the GPCM works well on a range of tasks, a consequence of the form of the kernel for $h$ is that $f$ has a rapidly decaying spectrum (see \cref{fig:priors_intro}), which means that $f$ will be a smooth function. If data are not (noisy) observations of a smooth function, the GPCM can fail to capture important structure and lead to predictions which are too smooth and consequently overshoot the data. 
In addition, the existing inference procedure by \citet{Tobar:2015:Learning_Stationary} relies on a mean-field assumption and is computationally expensive because it requires numerical optimisation over high-dimensional covariance matrices.

The purpose of this paper is twofold: to redesign the GPCM for non-smooth signals and to improve inference in terms of both approximation quality and computational expense.
Our contributions are as follows.
First, we propose two variations of the GPCM which induce a richer distribution over the spectrum with relaxed assumptions about smoothness: the Causal Gaussian Process Convolution Model (CGPCM) introduces a causality assumption into the GPCM, and the Rough Gaussian Process Convolution Model (RGPCM) can be interpreted as a non-parametric generalisation of the fractional Ornstein--Uhlenbeck process.
Second, we propose an improved variational inference scheme which goes beyond the mean-field assumption. In particular, we design a Gibbs sampler which directly samples from the optimal variational solution, which entirely circumvents any variational optimisation;
the Gibbs sampler is found to mix quickly and give uncertainty estimates superior to approaches that apply explicit optimisation. 
Finally, we validate the proposed variations of the GPCM and inference scheme in experiments on synthetic and real-world data.
\vspace{-0.25em}

\hiddensection{GAUSSIAN PROCESS CONVOLUTION MODELS}
\label{sec:gpcms}

The GPCM admits two equivalent formulations.
The first formulation of the GPCM is a linear system excited by white noise with a nonparametric prior over the filter \citep{Tobar:2015:Learning_Stationary}. 
This formulation is useful because it shows how the GPCM is constructed and hence how it can be modified to adjust properties.
It also forms the basis for an approximate inference scheme.
The second formulation of the GPCM is as a GP with a nonparametric prior over the kernel, which is the interpretation that we are ultimately after.

Let $k_h$ be a kernel with finite trace, meaning that
$
    \int_{-\infty}^\infty k_h(\tau, \tau) \isd{\tau} < \infty.
$\footnote{The kernel $k_h$ is then said to be a trace class Hilbert--Schmidt 
kernel \citep{Lax:2002:Functional_Analysis}.}
Then the linear system formulation of the GPCM is given by the following model:
\begin{gather}
    x \sim \GP(0, \delta(t - t')), \quad
    h \sim \GP(0, k_h(t, t')),  \\
    \textstyle f(t)\cond h, x  = \int_{-\infty}^\infty h(t - \tau) x(\tau) \isd \tau, \hspace{1.5em} \label{eq:gpcm_conv}
\end{gather}
where $\delta(\vardot)$ denotes the Dirac delta function.
This model is accompanied by the data likelihood 
$
    y\cond f\sim \GP(f(t), \sigma^2 \delta[t - t'])
$, where $\delta[\vardot]$ is the Kronecker\footnote{
    $\delta[0] = 1$ and $\delta[\vardot] = 0$ elsewhere.
} delta function.
From the requirement that $k_h$ has a finite trace it follows that $f$ has finite power:
$
    \Var[f(t)]
    = \smash{\int_{-\infty}^\infty} k_h(\tau, \tau) \isd{\tau} < \infty
$.
One family of kernels with finite trace is given by the \textit{amplitude-modulated, locally stationary} (AMLS) kernels \citep{Chen:2018:On_Kernel_Design_for_Regularized}
$
    k_h(t, t') = w(t) w(t') k_g(t - t')
$
where $w$ is square integrable and $k_g$ a stationary kernel.
Indeed, then
$
    \Var[f(t)]
    = k_g(0)\int_0 w^2(\tau) \isd{\tau} < \infty.
$
Choosing $k_h$ to be an AMLS kernel corresponds to a generative model where we first draw a filter $g$ and then apply a window function $w$ to obtain $h$: $g \sim \GP(0, k_g)$ and  $h(t)\cond g = w(t) g(t)$.
The window $w$ also serves to make inference well posed:
since $x$ is stationary, 
any shifted version of $h$ results in an identical model for $f\cond h$, but versions of $h$ contained within the window are preferred.
Following \citet{Tobar:2015:Learning_Stationary}, we choose $w(t) = e^{-\alpha t^2}$ and  $k_g(t - t') = e^{-\gamma(t - t')^2}$.
With these choices, $k_h(t,t') = e^{-\alpha t^2 - \alpha t^{\prime2} - \gamma(t-t')^2}$, which \citeauthor{Tobar:2015:Learning_Stationary} name the \textit{decaying exponentiated quadratic} (DEQ) kernel.
For the DEQ kernel, $\alpha$ determines the temporal extent of the filter $h$ and $\gamma$ the time scale on which the filter $h$ varies.

When conditioned on $h$, $f$ is a fixed linear transform ($h$ is fixed) of the Gaussian process $x$. Consequently, $f\cond h$ is also a Gaussian process, with zero mean, $\E[f(t) | h] = 0$, and covariance $k_{f \cond h}(t, t') = \E[f(t)f(t') | h]$.
This reveals an equivalent formulation of the GPCM with a nonparametric prior over the kernel:

\begin{model}[GPCM]
    \label{mod:gpcm_nonparametric_kernel}
    Let $k_h$ be a DEQ kernel.
    Then the GPCM is given by the following generative model:
    \begin{align}
        h &\sim \GP(0, k_h(t, t')), \\
        f\cond h &\sim \textstyle \GP(
            0,
            \int_{-\infty}^\infty h((t  - t') + \tau)h(\tau) \isd \tau
        ). \label{eq:gpcm_sampling}
    \end{align}
\end{model}
\vspace{-0.5em}

Observe that, in \cref{mod:gpcm_nonparametric_kernel}, the parametrisation of the kernel $k_{f \cond h}$ of $f$ is precisely the functional analogue of the parametrisation of a covariance matrix  with the outer product: $\mSigma = \mA \mA^\T$.
As alluded to in the introduction, a consequence of the strong smoothness assumptions on $h$, which derive from the DEQ kernel, is that the GPCM also exhibits smoothness:

\begin{proposition} \label{prop:smoothness_gpcm}
    Sample paths of the GPCM are almost surely everywhere differentiable.
    See \cref{app:proof_gpcm}.
\end{proposition}

\paragraph{The Causal GPCM.}
The linear system formulation of the GPCM in \eqref{eq:gpcm_conv} is an \emph{acausal} system, meaning that past system responses can depend on future inputs.\footnote{
    Although similarly named, the system-theoretic notion of causality here is distinct from the probabilistic notion of causality.
}
Acausality combined with the smoothness of $h$ lies at the heart of the smoothness of the GPCM.
Therefore, to build a model which is less smooth, we adjust the convolution in \eqref{eq:gpcm_conv} to be \emph{causal}.
In a causal system, a system response can only depend on past inputs, not  future inputs, which is in line with physical systems.
Following the construction of the GPCM  gives rise to the \emph{Causal GPCM} (see \cref{app:cgpcm}):

\begin{model}[CGPCM]
    \label{mod:cgpcm_nonparametric_kernel}
    Let $k_h$ be a DEQ kernel.
    The CGPCM is given by the following generative model:
    \begin{align}
        h &\sim \GP(0, k_h(t, t')), \\
        f\cond h &\sim \textstyle \GP(
            0,\int_0^\infty h(|t  - t'| + \tau)h(\tau) \isd \tau
        ). \label{eq:cgpcm_sampling}
    \end{align}
\end{model}
\vspace{-0.5em}

The only difference between \cref{mod:cgpcm_nonparametric_kernel} and \cref{mod:gpcm_nonparametric_kernel} is that the integral starts at zero and depends on $|t - t'|$ rather than on $t - t'$.
As the next proposition shows, this seemingly minor detail has major consequences for the smoothness properties of the CGPCM.

\begin{proposition} \label{prop:smoothness}
    If $h(0) = 0$, then sample paths of the CGPCM are almost surely everywhere differentiable.
    If, on the other hand, $h(0) \neq 0$, then sample paths of the CGPCM are almost surely nowhere differentiable.
    See \cref{app:proof_cgpcm} for a proof.
\end{proposition}

\begin{figure*}[t]
    \centering
    \includegraphics[width=\textwidth]{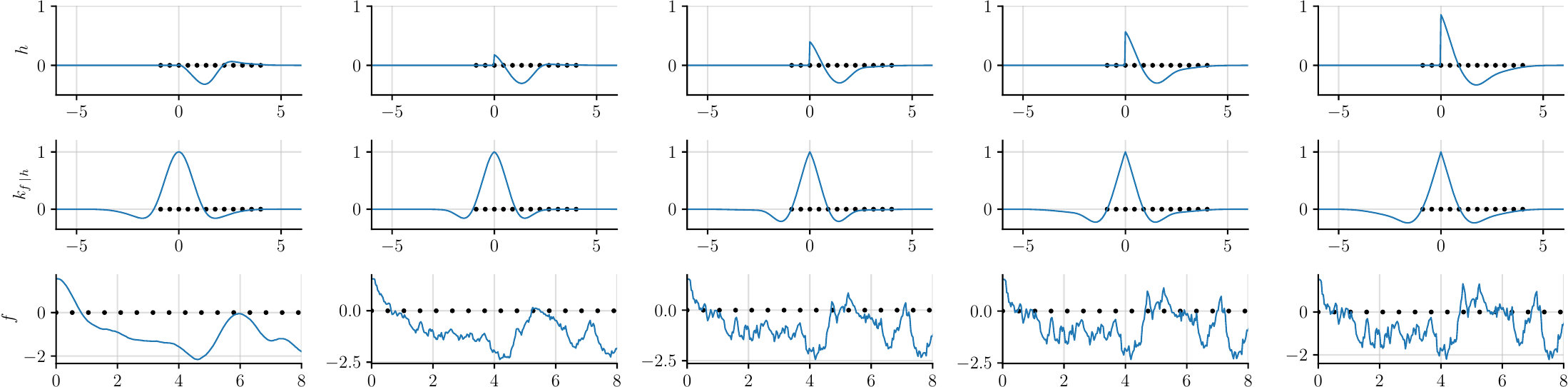}%
    \vspace{-0.35em}%
    \caption{
        Generative process of the CGPCM.
        Shows the filter $h$, the kernel $k_{f\cond h}$, and a sample $f\cond h \sim \GP(0,k_{f\cond h}(t-t'))$ while the filter is interpolated from one that satisfies $h(0)=0$ to one that satisfies $|h(0)|>0$.
        The sample appears smooth for $h(0)=0$ and becomes more irregular as $|h(0)|$ increases.
        The black dots indicate inducing point positions.
    }
    \vspace{-1em}
    \label{fig:interpolation}
\end{figure*}

Although \cref{prop:smoothness} tells us that sample paths of the CGPCM are almost surely nowhere differentiable in the case $\abs{h(0)} > 0$, the proposition does not give us a sense of how volatile the sample paths then are.
\Cref{app:cgpcm} argues $f$ that can locally be approximated by a $\abs{h(0)}$-scaled Brownian motion, which means that the magnitude of the irregular increments is controlled by the value of $\abs{h(0)}$.
Intuitively, $\abs{h(0)}$ is the magnitude of the filter when new white noise enters it: if $\abs{h(0)} > 0$, new noise is directly passed to the output, which results in a nondifferentiable signal; and the larger $\abs{h(0)}$ is, the more noisy the output will be.
\Cref{fig:interpolation} demonstrates this mechanism.
Since $h$ is modelled randomly, by performing inference in the CGPCM, the model is able to automatically infer a level of irregularity, \ie~a value for $\abs{h(0)}$, which is appropriate for the data.

\paragraph{The Rough GPCM.}
The CGPCM is able to model nondifferentiable phenomena.
This is visualised by samples in the bottom row of \cref{fig:interpolation}, which look fairly jagged. Some applications, however, may require samples which behave even more erratically, like equilibrium systems under noise in the natural sciences, and certain financial time series.
To this end, we modify both the filter $h$ and input $x$ of the (C)GPCM: we relax the smoothness of the filter $h$ for greater spectral flexibility, and posit a Mat\'ern--$\tfrac12$ kernel for the input process $x$. As will be explored in the next section, instead of using a smoothing transformation for inducing points for $x$, which the (C)GPCM use, 
our construction  will enable efficient inference of spectral content through  variational Fourier features.
The RGPCM can be interpreted as altering a Mat\'ern--$\tfrac12$ GP---also known as an Ornstein--Uhlenbeck (OU) process---by a random nonparametric modulation of the spectrum. A parametric special case of this model is the fractional Ornstein--Uhlenbeck process \citep{cheridito2003fractionalOU}, which can be \emph{rough} in the sense that it is more irregular than the OU process.\footnote{Formally, the OU process is Hölder continuous of  order $a< \smash{\tfrac12}$ while the fractional OU has $a<H$ for $H\in(0,1)$. If $H<\smash{\tfrac12}$, then it is called rough;
see \cite{gatheral2018volatility} and \cite{bennedsen2016decoupling} for empirical evidence of roughness in financial data.} We therefore call this version the \emph{Rough GPCM}.

\begin{model}[RGPCM]
    \label{mod:rgpcm_nonparametric_kernel}
    Let $h$ be white noise windowed by $w(t) = e^{-\alpha|t|}$ and 
    $k_x(t,t') = e^{-\lambda|t-t'|}$.
    Then the RGPCM is given by the following model:
    \begin{equation}
        h \sim \GP(0, k_h(t, t')),\;\;
        f\cond h \sim \GP\left(
            0, k_{f|h}(t-t')
        \right) 
    \end{equation} with
    $
    \textstyle
    k_{f\cond h}(r) \!=\! \int_0^\infty\!\!\int_0^\infty h(\tau)h(\tau') k_x(r\!-\!(\tau\!-\!\tau')) \isd\tau \isd\tau'
    $.
\end{model}

See \cref{app:rgpcm} for a more detailed description of the RGPCM.
As we will see next, the RGPCM allows for more spectral content over higher frequencies and thus more irregular sample paths.

\begin{figure*}[t]
    \centering
    \includegraphics[width=\textwidth]{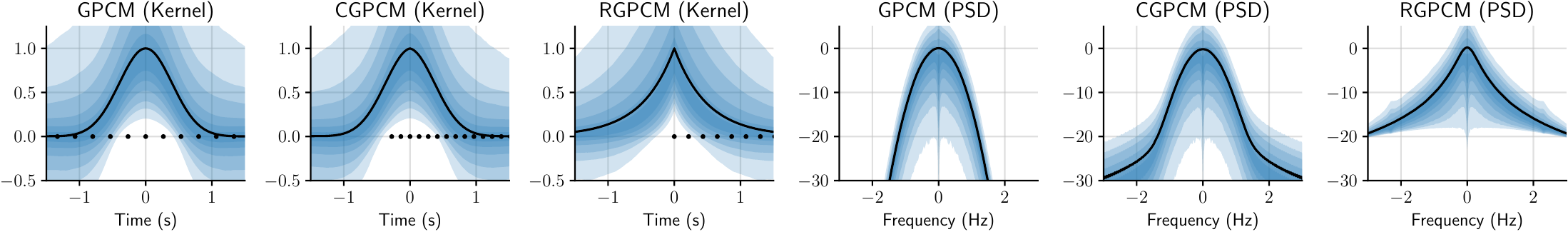}
    \vspace*{-1.5em}
    \caption{
        Visualisation of the nonparametric priors induced over kernels and PSDs by the GPCM, CGPCM, and RGPCM with $\tau_f = 0.5$ s and $\tau_w = 2$ s (see \cref{sec:inference}).
        The black line shows the mean and the shaded areas show marginal quantiles ranging from $1\%$ to $99\%$.
        The number of inducing points is $n_u = 30$; the black dots indicate inducing point positions.
    }
    \vspace{-0.5em}
    \label{fig:priors}
\end{figure*}

\paragraph{Comparison of model priors.}  
\Cref{fig:priors} visualises the nonparametric prior distribution over kernels and PSDs induced by the GPCM, CGPCM, and RGPCM.
Observe that the GPCM has a very quickly decaying spectrum, whereas the CGPCM and especially the RGPCM have substantial support at higher frequencies.
This is in line with the construction of the models: the GPCM models smooth signals, the CGPCM models signals with varying levels of irregularity, and the RGPCM models the most irregular signals.
To further support this, \cref{fig:samples} shows function, kernel, and PSD samples.
Crucially, observe that GPCM samples are smooth, the CGPCM varies in level of smoothness (\eg, the green sample is smooth and yellow one is jagged),
and the RGPCM is very irregular.

\begin{figure*}[t]
    \centering
    \includegraphics[width=\textwidth]{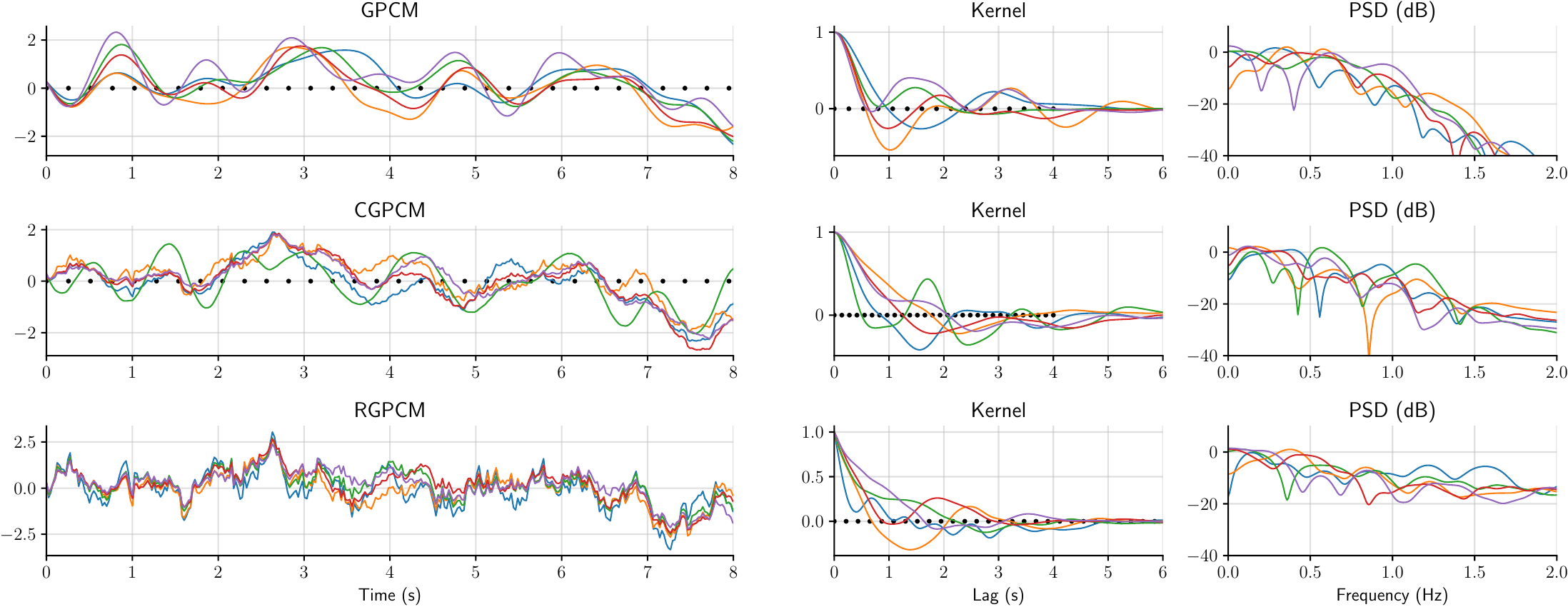}
    \vspace*{-1.5em}
    \caption{
        Prior function, kernel, and PSD samples from the GPCM, CGPCM, and RGPCM with $\tau_f = 0.5$ s and $\tau_w = 2$ s (see \cref{sec:inference}).
        The numbers of inducing points are $n_u = 30$ and $n_z = 40$;
        black dots indicate inducing point positions.
    }
    \vspace{-1em}
    \label{fig:samples}
\end{figure*}

\paragraph{Choice of hyperparameters.}
To fairly compare the GPCM, CGPCM, and RGPCM in experiments, we need to able to configure the models comparably.
By requiring the models' prior powers to be unity and that, for an appropriate definition of the length scale, the prior marginal covariance functions $\E[k_{f\cond h}(t, t')]$ have equal length scales, \cref{app:init} derives the following initialisation.
For some data, let $\tau_f$ be the smallest length scale contained in the signal and let $\tau_w$ be the desired extent of the filter.
Let the subscript $\vardot\ss{c}$ refer to the CGPCM, $\vardot\ss{ac}$ to the GPCM, and $\vardot\ss{r}$ to the RGPCM.
Then initialise $\alpha\ss{ac,c} = \frac{\pi}{4} \tau^{-2}_w$, $\alpha\ss{r} = \tau_w^{-1}$, $\gamma\ss{ac,c} =\frac{\pi}{4} {\tau_f^{-2}} - \frac12\alpha\ss{ac,c}$, and $\lambda\ss{r} = \tau_f^{-1}$.
By setting $\tau_w = 2 \tau_f$, $\tau_{f,\text{ac,c}} = \smash{\sqrt{\pi/2}}\ell\approx1.2 \ell$, and $\tau_{f,\text{rv}}=\ell$, we define the  \emph{standardised marginal covariance functions}:
\begin{align}
    k\ss{ac}(r) &= \exp(-\tfrac{1}{2\ell^2} r^2), \label{eq:standard-ac} \\
    k\ss{c}(r) &= (1 - \erf(\smash{\tfrac1{4\ell}}|r|)) \exp(-\tfrac{1}{2\ell^2} r^2), \label{eq:standard-c} \\
    k\ss{r}(r) &= \exp(-\tfrac1{\ell} |r|). \label{eq:standard-r} 
\end{align}
To the best of our knowledge, the kernel $k\ss{c}$ does not have an established name.
(That $k\ss{c}$  is a positive definite function follows from the fact that it is a covariance function of the CGPCM.)
We call $k\ss{c}$ the \emph{causal exponentiated quadratic} (CEQ) kernel.
The standard kernels \cref{eq:standard-ac,eq:standard-c,eq:standard-r} are helpful, because they allow us to build intuition for the GPCM models by comparing to familiar kernels:
the GPCM is like an EQ GP, the CGPCM is also like an EQ GP but with an irregular component, and the RGPCM is like a Mat\'ern--$\frac12$ GP.

\vspace{-0.25em}
\hiddensection{INFERENCE}
\label{sec:inference}

For all GPCM models, the posterior conditioned on data cannot be computed analytically, so an approximation is necessary.
We follow \citet{Tobar:2015:Learning_Stationary} and consider a variational approximation \citep{Wainwright:2008:Graphical_Models_Exponential_Families_and}.
The setup and derivation of our inference scheme is spelled out in detail in \cref{app:inference};
in this section, we  give a high-level sketch.

To approximate the posterior over the Gaussian processes $x$ and $h$, we make use of inducing points \citep{Titsias:2009:Variational_Learning, Matthews:2016:On_Sparse_Variational}.
For the GPCM and CGPCM, let $\vu$ be $n_u$ inducing points for $h$, and let $\vz$ be $n_z$ inducing points for the inter-domain transform $s(\tau) = \int_{-\infty}^\infty e^{-\omega (\tau - t)^2} x(t) \isd t$ \citep{Lazaro-Gredilla:2009:Inter-Domain_Gaussian_Processes_for_Sparse}.
Note that $s$ has an EQ kernel.
For the RGPCM, first, let $\vu$ be $n_u$ inducing points for the \emph{causal} inter-domain transform $s(\tau) = \int_{-\infty}^\tau e^{-\gamma (\tau - t)} x(t) \isd t$.
Compared to the GPCM, $s$ now has a Mat\'ern--$\frac12$ kernel.
Second, for $h$, we exploit the fact that $x$ also has a Mat\'ern--$\tfrac12$ kernel, which enables us to \emph{variational Fourier features} (VFFs) \citep{Hensman:2018:Variational_Fourier_Features_for_Gaussian}.
Whereas regular inducing points approximate the posterior with temporally local basis functions placed at the inducing points, VFFs approximate the posterior with a truncated Fourier series, which can give superior spectral approximation qualities.
Because $x$ is stationary, such an approximation is not possible with the Fourier transform as inter-domain transform;
rather, VFFs propose a clever construction which exploits the fact that harmonics are contained within the reproducing kernel Hilbert space of the Mat\'ern--$\tfrac12$ kernel.
See \citet{Hensman:2018:Variational_Fourier_Features_for_Gaussian} for more details.
For the RGPCM, we let $\vz$ be $n_z$ VFFs for $x$. 

Henceforth, let $\th$ denote all hyperparameters of a model. Given the inducing points $\vu$ and $\vz$, we follow \citeauthor{Tobar:2015:Learning_Stationary} and consider the variational approximation
$
    q_\th(h, x, \vu, \vz)
    = p_\th(h\cond \vu) p_\th(x \cond \vz) q(\vu, \vz)
$
where $q(\vz, \vu)$ is a joint variational approximation over the inducing points.
Given some data $\vy$, to optimise $q(\vz, \vu)$ and $\theta$, we optimise the \emph{evidence lower bound} (ELBO):
\begin{align}\label{eq:elbo}
    &\F_\th[q(\vu, \vz)]  \\
    &= \E_{q}[\log p_\th(\vy\cond f)]
        - \operatorname{KL}[q(\vu, \vz)\divsep p_\th(\vu)p_\th(\vz)]. \nonumber
\end{align}
As we explain in \cref{app:inference}, key to the tractability of the ELBO is the observation that $\E[\log p_\th(\vy \cond f)\cond \vu, \vz]$ is tractable and conditionally quadratic in $\vu$ and $\vz$.

\paragraph{Mean-field inference.}
The first scheme that we consider is the \emph{mean-field \emph{(MF)} approximation} $q(\vu, \vz) = q(\vu) q(\vz)$, originally considered by \citet{Tobar:2015:Learning_Stationary}.
They parametrise $q(\vu)$ and $q(\vz)$ by Gaussians with dense covariance matrices and optimise the ELBO using gradient-based optimisation.
In \cref{app:inference}, we show that, given $q(\vz)$ (resp.\ $q(\vu)$), the optimal $q^*(\vu)$ (resp.\ $q^*(\vz)$) can be computed analytically, which gives rise to a coordinate ascent (CA) scheme.
Alternatively, the optimal form $q^*(\vz)$ (resp.\ $q^*(\vu)$) can be plugged back into the ELBO to give rise a collapsed MF bound, which depends on many fewer variational parameters and hence greatly accelerates optimisation.

\paragraph{Structured inference.}
A major issue with the mean-field approximations is that it is unable to model correlations between $\vu$ and $\vz$.
It consequently biases towards overly simple models and and tends to yield poorly calibrated uncertainties \citep{MacKay:2002:Information_Theory_Learning,Turner:2011:Two_Problems_With_Variational_Expectation}.
To improve upon the mean-field approximation, we  consider a general  \emph{structured approximation} $q(\vu, \vz)$.
In \cref{app:inference}, we derive the optimal $q^*(\vu, \vz)$ and demonstrate that the conditionals $q^*(\vu \cond \vz)$ and $q^*(\vz \cond \vu)$ are Gaussians with parameters dependent on respectively $\vz$ and $\vu$.
Therefore, to sample from the optimal $q^*(\vu, \vz)$, we can iteratively sample from these conditionals in an alternating fashion.
This gives us a way to perform inference in the models without any variational optimisation, and which even enjoys computational benefits compared to the mean-field schemes (see \cref{app:inference}).
To optimise $\th$, \cref{app:inference} shows that samples from $q^*(\vu, \vz)$ can be used to approximate $\frac{\sd}{\sd\theta}\mathcal{F}_\th[q^*(\vu, \vz)]$.
Although gradients can be approximated, $\mathcal{F}_\th[q^*(\vu, \vz)]$ unfortunately cannot be estimated.
As a proxy, we propose the lower bound $\mathcal{F}_\th[q^*\ss{MF}(\vu) q^*(\vz \cond \vu)] \le \mathcal{F}_\th[q^*(\vu, \vz)]$, which can  be estimated. Here $q^*\ss{MF}(\vu)$ is the optimal MF solution.

\hiddensection{EXPERIMENTS}
\label{sec:experiments}

We provide an implementation of the GPCM, CGPCM, and RGPCM at \href{https://github.com/wesselb/gpcm}{\texttt{github.com/wesselb/gpcm}} with a user-friendly \texttt{sklearn}-style 
interface.
In this section, we
validate the proposed variants of the GPCM and inference scheme on synthetic and real-world data.
We  use mean log loss (MLL)\footnote{
    For predictions given by means and marginal variances, the mean log loss is the average negative log-pdf of the observations under those means and marginal variances.
}  as the metric to evaluate uncertainty and the root-mean-square error (RSME) as the metric to evaluate accuracy of the mean prediction.
Unlike \citet{Tobar:2015:Learning_Stationary}, in all experiments we optimise the inducing point locations and all hyperparameters.

\paragraph{Learning performance of inference schemes.}
We compare the inference schemes described in \cref{sec:inference} to the original inference scheme by \citet{Tobar:2015:Learning_Stationary}.  \Cref{fig:elbos} shows the evolution of the ELBO over wall-clock time for all inference schemes in a toy problem.
Note that the collapsed MF bound optimises quicker than the uncollapsed MF bound, but that the coordinate ascent procedure (CA) converges nearly instantaneously and reaches its convergence threshold long before the gradient-based optimisation.
Moreover, this example only uses $40$ inducing points. The benefits of the CA scheme will be further exaggerated for greater numbers of inducing points, because gradient-based optimisation will then struggle with the large covariance matrices.
Finally, observe that the structured scheme, which derives from the CA solution (see \cref{sec:inference}), further improves the ELBO and comes closest to the GP likelihood out of all inference schemes.

\begin{figure}[t]
    \centering
    \vspace{0.5em}
    \includegraphics[width=.75\linewidth]{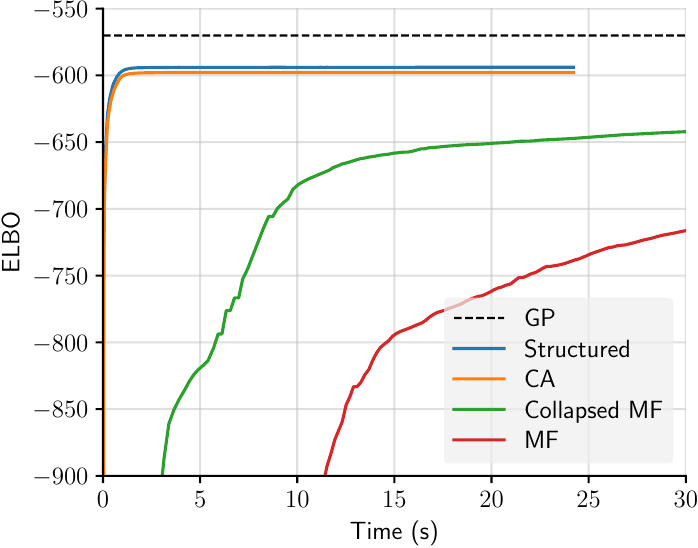}
    \vspace{-0.5em}
    \caption{
        Learning 500 noisy data points sampled from a GP with an EQ kernel.
        Shows the evolution of the ELBO in time for the original mean-field approximation by \citet{Tobar:2015:Learning_Stationary} (MF), a collapsed version of the mean-field approximation (collapsed MF), the coordinate-ascent version of the mean-field approximation (CA), and the structured ELBO with $q(\vu)$ given by the current CA solution;
        see \cref{sec:inference}.
        Also shows the likelihood of the GP from which the data was sampled.
        \texttt{scipy}'s implementation of the L-BFGS-B algorithm 
        \citep{Nocedal:2006:Numerical_Optimisation}
        was used to optimise the uncollapsed and collapsed mean-field ELBO.
        The numbers of inducing points are $n_u = 40$ and $n_z = 40$.
    }
    \vspace{-1em}
    \label{fig:elbos}
\end{figure}

\paragraph{Approximation quality of inference schemes.}
In this experiment, a reference to the mean-field approximation scheme will refer to the CA scheme, followed by optimisation of the hyperparameters with the collapsed MF ELBO, followed by another application of the CA scheme (see \cref{sec:inference}).
\Cref{fig:smk} presents the results when we use the GPCM to infer the kernel and PSD of a GP with a one-component spectral mixture kernel \citep{Wilson:2013:Spectral_Mixture} from a noisy sample.
Observe  that the  mean-field  scheme  produces a poor solution for the predictive mean and that the uncertainties are uncalibrated;
in contrast, the structured scheme is able to  capture true kernel and PSD.
We further compare the mean-field and structured approximation in a second experiment.
\Cref{fig:comparison} presents the results of using the GPCM, CGPCM, and RGPCM to infer their respective standard kernels from a noisy sample.
Observe that, in all cases, the structured scheme presents marked improvements in MLL;
indeed, \cref{fig:comparison} shows that the true kernels hit the edges or even just fall outside of the uncertainty intervals produced by the mean-field scheme.
In this case, however, the structured scheme did not yield improvements in RMSE.
Generally, the benefit of the structured scheme constitutes improved uncertainty estimates.
Having demonstrated the advantages of the structured scheme, we commit to the structured scheme in the remaining experiments with real data.

\begin{figure}[t]
    \small
    \centering
    \includegraphics[width=\linewidth]{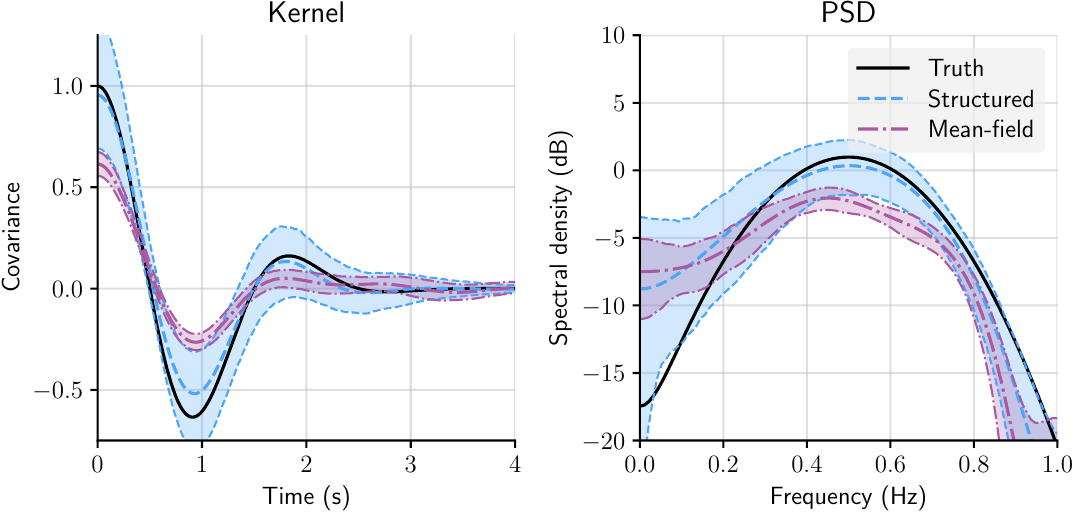} \\[0.5em]
    \begin{tabular}{llcccc}
\toprule
\textsc{Model}  & \textsc{Data} & \multicolumn{2}{c}{\textsc{MLL}} & \multicolumn{2}{c}{\textsc{RMSE}} \\
 &  & \textsc{MF} & \textsc{S} & \textsc{MF} & \textsc{S}\\ \midrule
\textsc{GPCM} & \textsc{SMK} & $34.8$ & $\mathbf{-1.40}$ & $0.19$ & $\mathbf{0.06}$ \\
\bottomrule
    \end{tabular}
    \caption{
        Fitting the GPCM on 200 noisy observations drawn from a GP with a one-component spectral mixture kernel using the mean-field inference scheme (MF) and structured inference scheme (S).
        The numbers of inducing points are $n_u = 80$ and $n_z = 80$.
        Shows the prediction for the kernel and for the PSD.
        Also shows the MLL and RMSE of the kernel prediction for both inference schemes.
        Best numbers are boldfaced.
    }
    \vspace{-1em}
    \label{fig:smk}
\end{figure}

\begin{figure}[t]
    \centering
    \small
    \includegraphics[width=\linewidth]{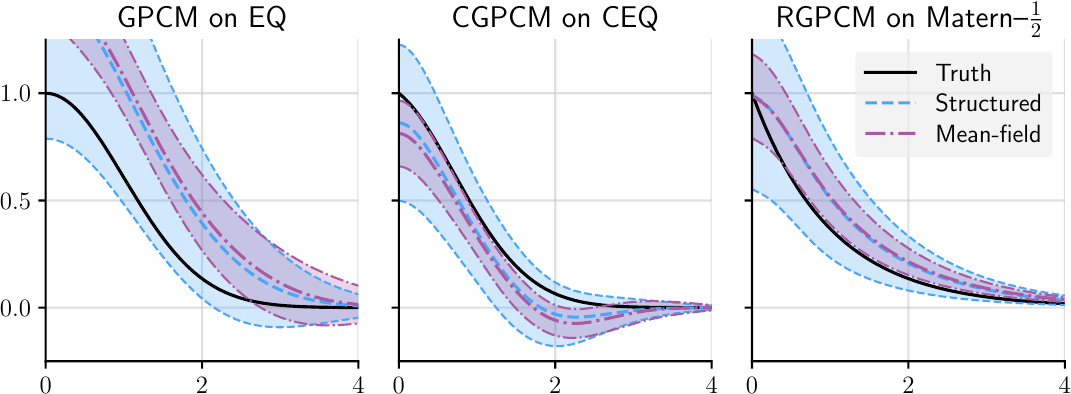} \\[1em]
    \begin{tabular}{llcccc}
\toprule
\textsc{Model}  & \textsc{Data} & \multicolumn{2}{c}{\textsc{MLL}} & \multicolumn{2}{c}{\textsc{RMSE}} \\
 &  & \textsc{MF} & \textsc{S} & \textsc{MF} & \textsc{S}\\ \midrule
\textsc{GPCM} & \textsc{EQ} & $\hphantom{-}4.43$ & $\mathbf{-0.61}$ & $0.34$ & $\mathbf{0.30}$ \\
\textsc{CGPCM} & \textsc{CEQ} & $-0.35$ & $\mathbf{-1.93}$ & $0.10$ & $\mathbf{0.08}$ \\
\textsc{RGPCM} & \textsc{Mat–$\frac{1}{2}$} & $\hphantom{-}0.57$ & $\mathbf{-1.32}$ & $\mathbf{0.10}$ & $\mathbf{0.10}$ \\
\bottomrule
        \end{tabular}
    \caption{
        Fitting the GPCM, CGPCM, and RGPCM on 400 noisy observations drawn from a GP with respectively a CEQ (see \cref{sec:inference}), EQ, and Mat\'ern--$\frac12$ kernel using the mean-field inference scheme (MF) and structured inference scheme (S).
        The numbers of inducing points are $n_u = 30$ and $n_z = 80$.
        Shows the prediction for the kernel and the MLL and RMSE for both inference schemes. Numbers within $1\%$ of the best number are boldfaced.
    }
    \label{fig:comparison}
\end{figure}

\paragraph{Predicting crude oil prices.}
We demonstrate the benefits of the relaxed smoothness assumptions of the CGPCM and RGPCM by, for the years 2012--2017, predicting  NASDAQ crude oil daily prices\footnote{\url{https://www.nasdaq.com/market-activity/commodities/cl\%3Anmx}} in every odd week of the second half of the year from all other data points in that year.
\Cref{fig:crude_oil} presents the results.
To begin with, we focus on the predictions by the models in the top two plots.
Observe that the predictions by the GPCM (blue) are much smoother than the predictions by the CGPCM (purple) and RGPCM (green).
The smoothness of the GPCM's predictions causes the model to explain more intricate structure of the signal as noise, which consequently leads to a significant increase in MLL.
That the predictions of the GPCM are too smooth is corroborated by the predictions for the PSD: the GPCM predicts a quickly decaying EQ-like spectrum, whereas the CGPCM and RGPCM predict support at higher frequencies and exhibit more fine-grained spectral structure.

\begin{figure*}[t]
    \vspace{-0.25em}
    \small
    \centering
    \includegraphics[width=\linewidth]{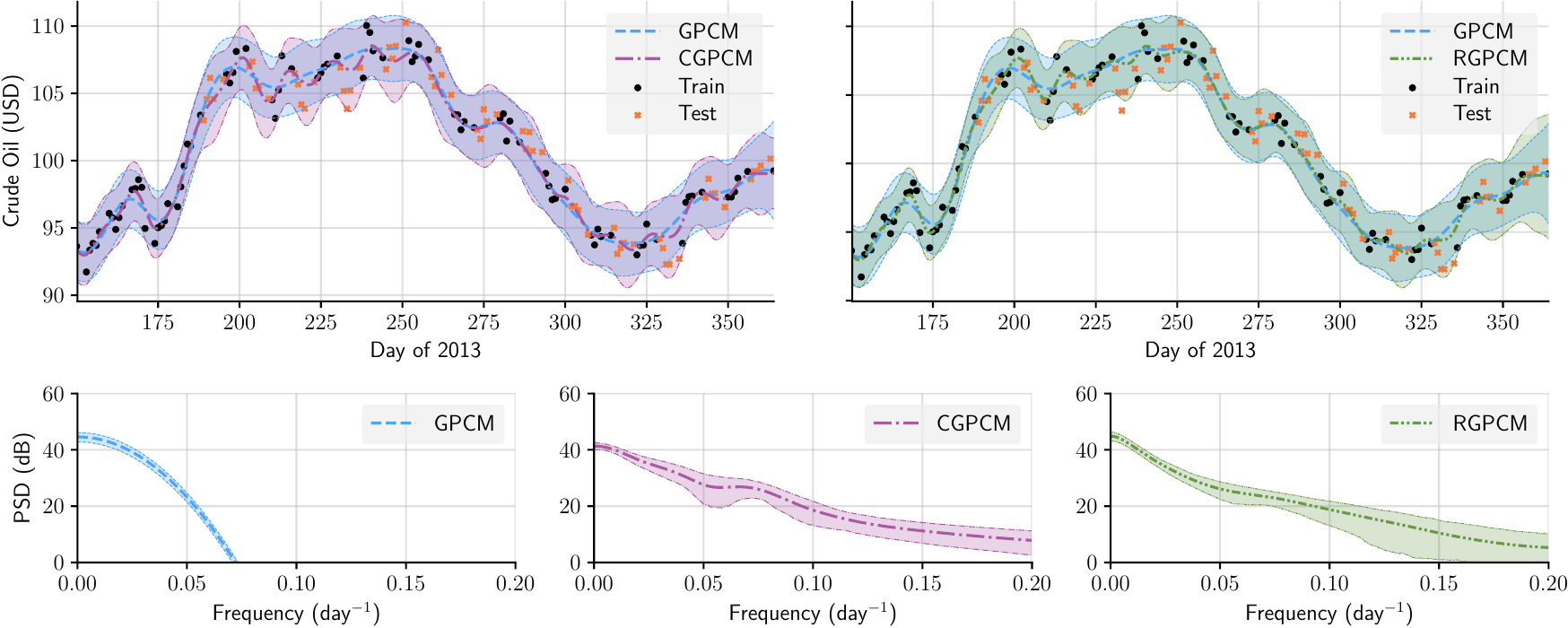} \\[0.75em]
    \hspace{1.5em}~\hfill
    \begin{tabular}{lccc}
        \toprule
         & \textsc{GPCM} & \textsc{CGPCM} & \textsc{RGPCM} \\ \midrule
        MLL
        & $1.765 {\scriptstyle\pm 0.077}$
        & $1.718 {\scriptstyle\pm 0.059}$
        & $\mathbf{1.681}^{*} {\scriptstyle\pm 0.061}$ \\
        \bottomrule
    \end{tabular}
    \hfill
    \begin{tabular}{lccc}
        \toprule
         & \textsc{GPCM} & \textsc{CGPCM} & \textsc{RGPCM} \\ \midrule
        RMSE
        & $1.426{\scriptstyle\pm 0.107}$
        & $1.358 {\scriptstyle\pm 0.072}$
        & $\mathbf{1.320} {\scriptstyle\pm 0.056}$ \\
        \bottomrule
    \end{tabular}
    \hfill~
    \caption{
        Predictions for NASDAQ crude oil prices by the GPCM, CGPCM, RGPCM for every odd week in the second half of 2013
        (see \cref{sec:experiments}).
        Tables show the MLL and RMSE of these predictions.
        Also shows the predictions for the PSDs.
        Models have $n_u = 50$ and $n_z = 150$.
        Errors are standard deviations.
        Best numbers are boldfaced.
        $\;{}^*$Significantly better than the GPCM with $p < 0.05$ using a paired test.
    }
    \vspace{-0.5em}
    \label{fig:crude_oil}
\end{figure*}
\begin{figure*}[t]
    \centering
    \includegraphics[width=.95\linewidth]{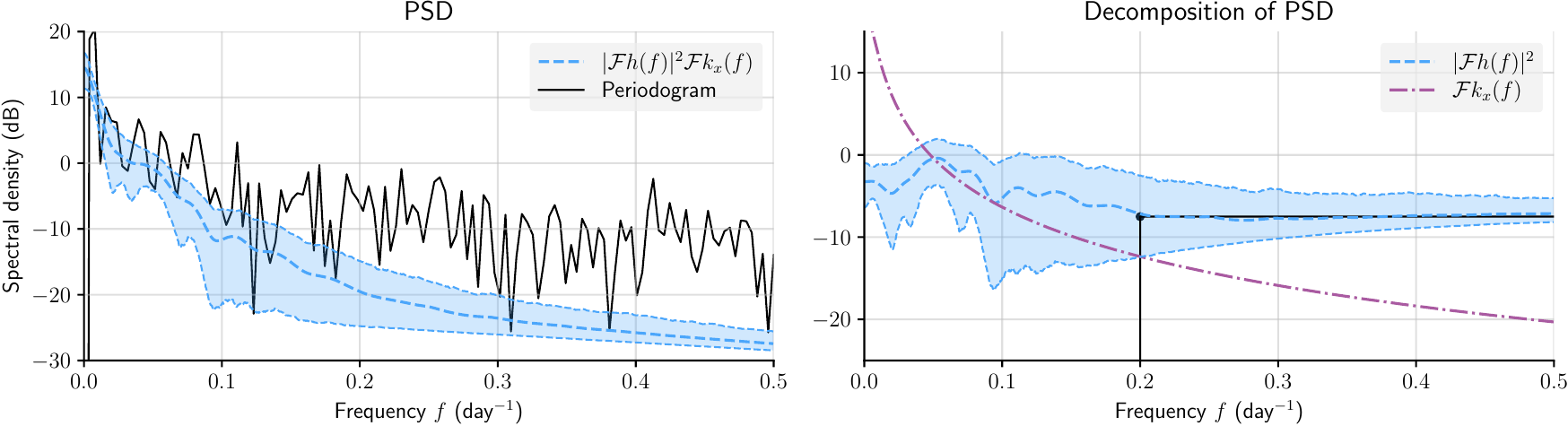}
    \caption{
        Prediction for the PSD of $\log$-VIX from the year 2000 by the RGPCM (see \cref{sec:experiments}).
        Shows the periodogram, the whole prediction $|\mathcal{F}h(f)|^2\mathcal{F}k_x(f)$ for the PSD, the Mat\'ern--$\frac12$ part $\mathcal{F}k_x(f)$ of the predicted PSD, and the modulation by the filter $|\mathcal{F}h(f)|^2$.
        We set $n_u = 60$ and $n_z$ equal to the number of time points, allowing the RGPCM to detect fluctuations up to the Nyquist frequency.
        Observe that the prediction for $|\mathcal{F}h(f)|^2$ is flat after $f = 0.2$ day${}^{-1}$.
    }
    \vspace{-1em}
    \label{fig:vix-psd}
\end{figure*}

\paragraph{Forecasting the Cboe Volatility Index.}
The Cboe Volatility Index\footnote{\url{https://www.cboe.com/tradable_products/vix/vix_historical_data/}} (VIX) is an index which measures the market's expectations for short-term S\&P500 price changes.
In this experiment, we train the models on the year 2015, retain the posterior over $h$, move forward a year and, in a one-week rolling window fashion, for 100 weeks, predict $\log$-VIX one week ahead given the past four weeks.
\Cref{tab:vix-prediction} presents the results.
Whereas the CGPCM outperforms the GPCM, which means that the causality assumption is helpful, the even weaker smoothness assumptions of the RGPCM yield the best uncertainty and mean estimates.

\begin{table}[t]
    \small
    \centering
    \begin{tabular}{l@{\hspace{4pt}}c@{\hspace{4pt}}c@{\hspace{4pt}}c}
        \toprule
         & \textsc{GPCM} & \textsc{CGPCM} & \textsc{RGPCM} \\ \midrule
        MLL & $\hphantom{-}0.089^{\hphantom{*}} {\scriptstyle\pm 0.087}$ & $-0.397^{*} {\scriptstyle\pm 0.070}$ & $\mathbf{-0.464}^{*} {\scriptstyle\pm 0.072}$\\
        RMSE & $\hphantom{-}0.210^{\hphantom{*}} {\scriptstyle\pm 0.010}$ & $\hphantom{-}0.143^{*} {\scriptstyle\pm 0.007}$ & $\hphantom{-}\mathbf{0.134}^{*} {\scriptstyle\pm 0.007}$ \\
        \bottomrule
    \end{tabular}
    \caption{
        Average one-week-ahead prediction result for $\log$-VIX (see \cref{sec:experiments}).
        Shows the MLL and RSME for the GPCM, CGPCM, and RGPCM. Best numbers boldfaced.
        Errors are standard deviations.
        $\;{}^*$Significantly better than all worse scores with $p < 10^{-6}$ using a paired test.
    }
    \vspace{-1.5em}
    \label{tab:vix-prediction}
\end{table}

\paragraph{Analysing the Cboe Volatility Index.}
In the final experiment, we use the RGPCM to investigate on which length scale the $\log$-VIX reasonably be approximated by an OU process.
We fit the RGPCM to all $\log$-VIX in the year 2000 using a number of inducing points which allows the RGPCM to detect fluctuations up to the Nyquist frequency.
We then use the property of the RGPCM that it is a spectrally modulated version of an OU process (see \cref{sec:inference} and \cref{app:rgpcm}):
the prediction of the PSD $|\mathcal{F}h(f)|^2\mathcal{F}k_x(f)$ by the RGPCM, where $\mathcal{F}$ denotes the Fourier transformation, can be decomposed into the prediction of the spectrum $\F k_x(f)$ for the OU process $x$ and a prediction of the modulation by the filter $|\mathcal{F}h(f)|^2$.
\cref{fig:vix-psd} shows that the prediction for $|\mathcal{F}h(f)|^2$ is flat after $f = 0.2$ day${}^{-1}$.
We conclude that the $\log$-VIX data can reasonably be modelled with an OU process on a length scale of at most $0.2^{-1} = 5$ days; on longer length scales, the RGPCM predicts more intricate spectral structure.

\hiddensection{RELATED WORK}
\citet{Tobar:2015:Inter-Domain_Inducing} extend the GPCM by considering an harmonic inter-domain transformation for $x$, making the model more suitable for more complex spectral estimation tasks.
In the supplement, \citet{Tobar:2015:Learning_Stationary} point out that the PSD of $f \cond \vu$ is a mixture of Gaussians centred at frequencies determined by the inducing point locations.
From this point of view, the GPCM is related to the models by \citet{Oliva:2016:Bayesian_Nonparametric_Kernel-Learning}, who use a Dirichlet process to parametrise the PSD; by \citet{Jang:2017:Scalable_Levy_Process_Priors_for}, who use a Levy process for the kernel, resulting in PSDs consisting of Laplacian mixtures; and by \citet{Benton:2019:Function-Space_Distributions_Over_Kernels}, who model the log-density of the PSD with a GP.
However, to perform inference, all these models employ general-purpose MCMC procedures;
in contrast, our inference scheme exploits the additional structure that the conditionals of the optimal variational solution can be sampled from directly to construct a Gibbs sampler which  mixes quickly in practice.
Finally, a construction like the GPCM can also be found in other fields:
\citet{Pillonetto:2010:A_New_Kernel-Based_Approach_for,Wagberg:2018:Regularized_Parametric_System_Identification_A,Chen:2018:On_Kernel_Design_for_Regularized} use a frequentist approach similar to GPCM for the purpose of system identification.

\hiddensection{DISCUSSION}
The goal of this paper was to redesign the GPCM to induce a richer distribution over the spectrum. 
We introduced the \emph{Causal GPCM}, able to model signals of varying level of irregularity, and the \emph{Rough GPCM}, able to model even more irregular signals.
The RGPCM is particularly appealing, because it avoids the implementation difficulties of the CGPCM and enjoys the benefits of variational Fourier features for improved approximation qualities.
Experiments demonstrated that the relaxed smoothness assumptions of the CGPCM and RGPCM can yield  substantially improved uncertainty and mean estimates on real-world data.
In addition, to address the deficiencies of the original mean-field inference scheme by \citet{Tobar:2015:Learning_Stationary}, 
we introduced a structured approximation which is able to fully retain the correlation structure between the latent variables.
Experiments showed that the structured scheme generally gives marked improvements in uncertainty estimates and can perform well in cases where the mean-field scheme falls over.

\section*{Acknowledgements}
We thank anonymous referees for helpful comments and discussions.
Wessel P. Bruinsma was supported by the Engineering and Physical Research Council (studentship number 10436152). 
Martin Tegn\'er thanks the Oxford-Man Institute for research support.
Richard E. Turner is supported by Google, Amazon, ARM, Improbable and EPSRC grant EP/T005386/1.
\bibliographystyle{unsrtnat}
\bibliography{bibliography}

\clearpage
\appendix

\thispagestyle{empty}
\onecolumn \makesupplementtitle

\tableofcontents
\clearpage

\section{Proof of Proposition \ref{prop:smoothness_gpcm}}
\label{app:proof_gpcm}

\begin{proof}[Proof of Proposition \ref{prop:smoothness_gpcm}.]
    To begin with, $h$ is a sample of a GP with an EQ kernel multiplied by a smooth window $w$, so $h$ is infinitely differentiable almost surely.
    Since, in addition, the window $w(t)=e^{-\alpha t^2}$ and all its derivatives decay to zero quickly, almost surely, (1) $h$ and all derivatives of $h$ go to zero at infinity and (2) any product of $h$ and its derivatives is dominated by an integrable function.
    This is argued more rigorously in \cref{app:well-behavedness-filter}.
    We will use these fact implicitly in the remainder of the proof to freely interchange integral and derivative.

    Note that
    \begin{equation}
        k_{f \cond h}(r)
        = k_{f \cond h}(0) + k_{f \cond h}'(0)r + \tfrac{1}{2}k_{f \cond h}''(0)r^2 + O(|r|^3),
    \end{equation}
    so sample paths of the GPCM are almost surely everywhere differentiable if $\smash{k_{f \cond h}'(0)} = 0$ and $\smash{k_{f \cond h}''(0)} \neq 0$ \citep[Thm 3 from][]{Cambanis:1973:On_Some_Continuity_and_Differentiability}.
    First, use integration by parts to find
    \begin{equation}
        k_{f \cond h}'(0)
        = \int^\infty_{-\infty} h'(\tau) h(\tau) \isd{\tau}
        = \cancel{\lim_{\tau \to \infty} [h^2(\tau) - h^2(-\tau)]} - \int^\infty_{-\infty} h(\tau) h'(\tau) \isd{\tau}
        = \vphantom{\int^\infty_{-\infty}} - k_{f \cond h}'(0)
    \end{equation}
    almost surely.
    Thus, $k_{f \cond h}'(0) = 0$ almost surely.
    Second, again use integration by parts to find
    \begin{equation}
        k_{f \cond h}''(0)
        = \int^\infty_{-\infty} h''(\tau) h(\tau) \isd{\tau}
        =  \cancel{\lim_{\tau \to \infty} [h'(\tau) h(\tau) - h'(-\tau) h(-\tau)]} - \int^\infty_{-\infty} (h'(\tau))^2 \isd{\tau}
        < 0
    \end{equation}
    almost surely, so $k_{f \cond h}''(0)\neq0$ almost surely.
\end{proof}

\subsection{Well-Behavedness of the Filter}
\label{app:well-behavedness-filter}
We show that the filter $h$ is well behaved.
In particular, we show that, almost surely, (1) $h$ and all derivatives of $h$ go to zero at infinity and (2) any product of $h$ and its derivatives is dominated by an integrable function.

\begin{definition}
    A function $w\colon \R \to \R$ is said to \textit{decay (sufficiently) quickly} if there exist $p_i \ge 0$ and $\alpha_i, \beta_i > 0$ such that
    \begin{equation}
        |w(t)| \le \sum_{i=1}^n|t|^{p_i} \exp(-\alpha_i |t|^{\beta_i})
    \end{equation}
    for all $t \in \R$.
\end{definition}

\begin{lemma} \label{lem:summation}
    Let $w$ decay sufficiently quickly, and let $(t_n)_{n\ge1}\sub \R$, $t_n \uparrow \infty$ be such that $t_n \ge n$ for all $n$.
    Then $\sum_{n=1}^\infty n\, |w(t_n)| < \infty$.
\end{lemma}
\begin{proof}
    To begin with, note that
    \begin{equation}
        \sum_{n=1}^\infty n |t_n|^p e^{-\alpha t_n^\beta}
        \le \sum_{n=1}^\infty \floor{t_n} (\floor{t_n} + 1)^p \exp(-\alpha \floor{t_n}^\beta).
    \end{equation}
    Consider the sequence
    \begin{equation}
        a_n
        = n(n + 1)^p \exp(-\alpha n^\beta)
        \le 2^p n^{p + 1} \exp(-\alpha n^\beta).
    \end{equation}
    We claim that $\sum_{n=1}^\infty a_n < \infty$, where the convergence is absolute.
    Then $\sum_{k=1}^\infty a_{n_k} < \infty$ for any subsequence $(n_k)_{k\ge1} \sub\N$.
    In particular, $\sum_{k=1}^\infty a_{\floor{t_k}} < \infty$, which shows the result.
    To show the claim, let $m \in \N$ be such that $m \beta > p + 3$.
    Then the estimate $\exp(-x) \le m! x^{-m}$ with $x \ge 0$ gives
    \begin{equation}
        n^{p + 1} \exp(-\alpha n^\beta)
        \le n^{p + 1}\cdot m! \parens*{\alpha n^\beta}^{-m}
        = \frac{m!}{\alpha^m} n^{p + 1 - m \beta}
        \le \frac{m!}{\alpha^m} n^{-2}.
    \end{equation}
    Since $\sum_{n=1}^\infty n^{-2} < \infty$, indeed $\sum_{k=1}^\infty a_{n} < \infty$.
\end{proof}

\begin{proposition} \label{prop:decay}
    Let $f$ be a stationary Gaussian process and let $w$ be a function that decays sufficiently quickly.
    Then
    \begin{equation}
        \lim_{t\to\infty} w(t) f(t) = 0
    \end{equation}
    almost surely.
\end{proposition}
\begin{proof}
    For suppose not.
    Let $A$ be measurable set with $\P(A)>0$ on which
    \begin{equation}
        \limsup_{t\to\infty} |w(t)f(t)| = L > 0
    \end{equation}
    with $L=\infty$ allowed.
    Then there is a sequence $(t_n)_{n \ge 1}\sub\R$, $t_n \uparrow \infty$, such that
    \begin{equation}
        \lim_{n\to\infty} |w(t_n) f(t_n)| = L.
    \end{equation}
    Since $t_n \uparrow \infty$, we may assume that $t_n \ge n$ for all $n$ by passing to a subsequence.
    Set
    \begin{equation}
        B_n = \set{|w(t_n) f(t_n)| \ge (C_0 + n)^{-1}},
    \end{equation}
    where $C_0 > 0$ is chosen such that $C_0^{-1}<L$.
    By construction, $\P(B_n \text{ i.o.}) \ge \P(A) > 0$.
    We claim, however, that $\sum_{n=1}^\infty \P(B_n) < \infty$.
    Then, by the Borel--Cantelli lemma, $\P(B_n \text{ i.o.}) = 0$, which indeed is a contradiction.
    To show the claim, note that, by Markov's inequality,
    \begin{align}
        \P(B_n)
        &=\P\parens{|w(t_n) f(t_n)| \ge (C_0 + n)^{-1}} \\
        &\le (C_0 + n) |w(t_n)| \E(|f(t_n)|) \\
        &\le C_1 (C_0 + n) |w(t_n)|
    \end{align}
    where $C_1 = \E(|f(0)|)<\infty$.
    Thus $\sum_{n=1}^\infty \P(B_n) < \infty$, by \cref{lem:summation}.
\end{proof}

\begin{lemma}
    Let $f \ge 0$ and $g > 0$ be continuous.
    If
    \begin{equation}
        \lim_{t\to\infty} \frac{f(t)}{g(t)} = 0
        \quad\text{and}\quad\lim_{t\to-\infty} \frac{f(t)}{g(t)} = 0,
    \end{equation}
    then there exists a $C>0$ such that $C g$ dominates $f$.
\end{lemma}
\begin{proof}
    Let $R$ be such that $|t| \ge R$ implies that $f(t)/g(t) < 1$.
    Then $g$ dominates $f$ on $(-\infty, -R] \cup [R,\infty)$.
    Since $[-R, R]$ is compact and $f$ and $g$ are continuous, on $[-R, R]$, $f$ attains its maximum $M$ and $g$ its mimumum $m$, where $m > 0$ because $g > 0$.
    Setting $C = \max\set{1, M / m}$ then works.
\end{proof}

The filter $h$ is generated according to
\begin{equation}
    g \sim \GP(0, k_g), \qquad
    h(t)\cond g = w(t) g(t),
\end{equation}
where
\begin{equation}
    w(t) = \exp(-\alpha t^2), \qquad
    k_g(t - t') = \exp(-\gamma(t - t')^2).
\end{equation}
Here $g$ is stationary and, almost surely, has pathwise derivatives of all orders \citep[\eg, Theorem 4,][]{Cambanis:1973:On_Some_Continuity_and_Differentiability}.

Let $\beta$ be such that $0 < \beta < \alpha$.
Then, almost surely,
\begin{equation}
    \lim_{t\to\infty}\frac{|w(t)g(t)|}{\exp(-\beta t^2)}
    = 0
    \quad\text{and}\quad
    \lim_{t\to-\infty}\frac{|w(t)g(t)|}{\exp(-\beta t^2)}
    = 0
\end{equation}
because $t \mapsto \exp(\beta t^2) w(t) = \exp(-(\alpha - \beta)t^2)$ decays sufficiently quickly.
Thus, almost surely, there exists a $C>0$ such that $t\mapsto C \exp(-\beta t^2)$ dominates $w g$, and $t\mapsto C \exp(-\beta t^2)$ is integrable and goes to zero at infinity.

Note that $t \mapsto \exp(\beta t^2) |w^{(n)}(t)|$, where $w^{(n)}$ is the $n^{\text{th}}$ derivative of $w$, decays sufficiently quickly for all $n \in \N$.
For any derivative of $w g$, use the product rule to expand and argue similarly to obtain a dominating function also of the form $t\mapsto C \exp(-\beta t^2)$.

In conclusion, almost surely, $h$ and all derivatives of $h$ are dominated by integrable functions that go to zero at infinity, and any product of these dominating functions is integrable.
Therefore, almost surely, (1) $h$ and all derivatives of $h$ go to zero at infinity and, (2) any product of $h$ and its derivatives is dominated by an integrable function.

\section{Proof of Proposition \ref{prop:smoothness}}
\label{app:proof_cgpcm}
\begin{proof}[Proof of Proposition \ref{prop:smoothness}]
    The proof proceeds like the proof for \cref{prop:smoothness_gpcm}.
    Let
    \begin{equation}
        z(r) = \int^\infty_0 h(r + \tau) h(\tau) \isd{\tau}.
    \end{equation}
    Then
    \begin{equation}
        k_{f\cond h}(r)
        = z(|r|)
        = z(0) + z'(0)|r| + \tfrac{1}{2}z''(0)r^2 + O(|r|^3),
    \end{equation}
    so sample paths of the CGPCM are almost surely everywhere differentiable if $z'(0) = 0$ and $z''(0) \neq 0$ and almost surely nowhere differentiable if $z'(0) \neq 0$ \citep[Thms 3 and 4 from][]{Cambanis:1973:On_Some_Continuity_and_Differentiability}.
    First, use integration by parts to find
    \begin{equation}
        z'(0)
        = \int^\infty_0 h'(\tau) h(\tau) \isd{\tau}
        = \cancel{\lim_{\tau \to \infty} h^2(\tau)} - h^2(0) - \int^\infty_0 h(\tau) h'(\tau) \isd{\tau}
        = -h^2(0) - z'(0)
    \end{equation}
    almost surely.
    Thus $z'(0) = - \frac{1}{2} h^2(0)$ almost surely.
    Second, if $h(0)=0$, again use integration by parts to find
    \begin{equation}
        z''(0)
        = \int_0^\infty h''(\tau) h(\tau) \isd{\tau}
        = \cancel{\lim_{\tau \to \infty} h'(\tau) h(\tau)} - h'(0) h(0) - \int^\infty_0 (h'(\tau))^2 \isd{\tau}
        = - \int^\infty_0 (h'(\tau))^2 \isd{\tau} < 0.
    \end{equation}
    almost surely.
    Therefore, if $h(0)=0$, then $z'(0)=0$ and $z''(0)\neq0$ almost surely; and if $h(0)\neq 0$, then $z'(0)\neq0$ almost surely.
    The result now follows.
\end{proof}

\section{The Causal Gaussian Process Convolution Model}
\label{app:cgpcm}

\subsection{Equivalent Formulations}
\paragraph{Linear system formulation:}
Let $k_h$ be the the following DEQ kernel:
\begin{equation}
    k_h(t, t') = \tilde\alpha^2 e^{-\alpha t^2 - \alpha t^{\prime2} - \gamma(t - t')^2}.
\end{equation}
Then the linear system formulation of the CGPCM is given by the following generative model:
\begin{equation}
    f(t)\cond h, x  = \int_{-\infty}^t h(t - \tau) x(\tau) \isd \tau
    \quad \text{where} \quad
    x \sim \GP(0, \delta(t - t')), \quad
    h \sim \GP(0, k_h(t, t')), 
\end{equation}
where $\delta(\vardot)$ denotes the Dirac delta function.
Compared to the GPCM, the linear system formulation of the CGPCM uses a \emph{causal} convolution operation.

\paragraph{Nonparametric kernel formulation:}
To derive the equivalent nonparametric kernel formulation, note that
\begin{equation}
    \E[f(t) \cond h]
    = \int_{-\infty}^t h(t - \tau) \E[x(\tau)] \isd \tau
    = 0.
\end{equation}
Moreover,
\begin{equation} \label{eq:kernel_cpgcm_causal}
    \E[f(t)f(t') \cond h]
    = \int_{-\infty}^t \int_{-\infty}^{t'} h(t - \tau)h(t' - \tau') \E[x(\tau)x(\tau')] \isd \tau' \isd \tau \\
    = \int_{-\infty}^{t \land t'} h(t - \tau)h(t' - \tau) \isd \tau.
\end{equation}
If $t \le t'$, then
\begin{equation}
    \E[f(t)f(t') \cond h]
    = \int^t_{-\infty} h(t - \tau) h(t' - \tau) \isd \tau
    = \int_0^\infty h(\tau) h(t' - t + \tau) \isd \tau.
\end{equation}
Similarly, if $t \ge t'$, then
\begin{equation}
    \E[f(t)f(t') \cond h]
    = \int^{t'}_{-\infty} h(t - \tau) h(t' - \tau) \isd \tau
    = \int_0^\infty h(t - t' + \tau) h(\tau) \isd \tau.
\end{equation}
Therefore, in any case,
\begin{equation}
    \E[f(t)f(t') \cond h]
    = \int_0^\infty h(|t - t'| + \tau) h(\tau) \isd \tau. 
\end{equation}
We conclude that the CGPCM is equivalent to the following generative model:
\begin{equation}
    f \cond h \sim \GP\parens*{
        0, \int_0^\infty h(\abs{t - t'} + \tau) h(\tau) \isd \tau
    }, \quad
    h \sim \GP(0, k_h(t, t')).
\end{equation}

\subsection{Local Behaviour}
We argue that the irregularity of the sample paths of the CGPCM is controlled by the value of $\abs{h(0)}$.
Because $h$ is smooth, for small $\ep > 0$ and $\tau \in (0, \ep)$, $h(\tau) \approx h(0)$, which means that
\begin{equation}
    f(t + \ep) - f(t)
    = \int_{t}^{t + \ep} h(t - \tau) x(\tau) \isd \tau
    \approx  h(0) \int_{t}^{t + \ep} x(\tau) \isd \tau
    \overset{\text{d}}{=} h(0) B_{\ep}
\end{equation}
where $\smash{\overset{\text{d}}{=}}$ denotes equality in distribution and $(B_t)_{t \ge 0}$ is a standard Brownian motion.
Therefore, $f$ can locally be approximated by a $\abs{h(0)}$-scaled Brownian motion, which means that the magnitude of the irregular increments is controlled by the value of $\abs{h(0)}$.

\subsection{Inducing Points}
\label{app:inducing_points_cgpcm}
For the filter $h$, define $n_u$ inducing point inputs $\vt_{u}$ initialised to uniformly spaced over $[-2\Delta, 2 \tau_w]$ where $\Delta$ is the inter-point spacing and $\tau_w$ the extent of the filter (see \cref{sec:gpcms}). Let the inducing points $\vu$ then be
\begin{equation}
    \vu = (h(t_{u,1}), \ldots, h(t_{u,n_u})).
\end{equation}
For the excitation signal $x$, we first define the inter-domain transformation
\begin{equation}
    s(\tau) = \int_{-\infty}^\infty \tilde\omega e^{-\omega (\tau - t)^2} x(t) \isd t.
\end{equation}
Then define $n_z$ inducing point inputs $\vt_z$ initialised to uniformly spaced over a window containing the data and let the inducing points $\vz$ be
\begin{equation}
    \vz = (s(t_{z,1}), \ldots, s(t_{z,n_u})).
\end{equation}
With these choices, we have the covariances
\begin{align}
    k_\vu(t)
    &\coloneqq \vphantom{\int_{-\infty}^\infty} \E[h(t) h(\vt_u)] = k_h(t, \vt_u), \\
    \mK_\vu
    &\coloneqq \vphantom{\int_{-\infty}^\infty} \E[h(\vt_u) h(\vt_u)^\T] = k_h(\vt_u, \vt_u^\T), \\
    k_\vz(t)
    &\coloneqq \E[x(t) s(\vt_z)]
    = \int_{-\infty}^\infty \tilde\omega e^{-\omega (\vt_z - t')^2} \E[x(t) x(t')] \isd t'
    = \tilde\omega e^{-\omega (\vt_z - t)^2}, \\
    \mK_\vz
    &\coloneqq
    \E[s(\vt_z) s(\vt_z)^\T]
    = \tilde\omega^2 \sqrt{\frac{\pi}{2\omega}} e^{-\frac12 \omega(\vt_z - \vt_z^\T)^2}
\end{align}
where the expression for $\mK_\vz$ follows from
\begin{align}
    \E[s(\vt_z) s(\vt_z)^\T]
    &=\int_{-\infty}^\infty \int_{-\infty}^\infty \tilde\omega^2 e^{-\omega (\vt_z - t)^2-\omega (\vt^\T_z - t')^2} \E[x(t) x(t')] \isd t' \isd t \\
    &=\int_{-\infty}^\infty \tilde\omega^2 e^{-\omega (\vt_z - t)^2-\omega (\vt^\T_z - t)^2}  \isd t \\
    &= \tilde\omega^2 \sqrt{\frac{\pi}{2\omega}} e^{-\frac12 \omega(\vt_z - \vt_z^\T)^2}.
\end{align}

\section{The Rough Gaussian Process Convolution Model}
\label{app:rgpcm}

\subsection{Equivalent Formulations}
\paragraph{Linear system formulation:}
Let $k_h$ be the covariance function of white noise windowed by $w(t) = \tilde\alpha e^{-\alpha|t|}$:
\begin{equation}
    k_h(t, t') = \tilde\alpha^2 e^{-\alpha |t| - \alpha |t^{\prime}|}\delta(t - t')
\end{equation}
where $\delta(\vardot)$ denotes the Dirac delta function.
Let $k_x$ be the covariance function of an Ornstein--Uhlenbeck process with length scale $\lambda$:
\begin{equation}
    k_x(t, t') = e^{-\lambda |t - t'|.}
\end{equation}
Then the linear system formulation of the RGPCM is given by the following generative model:
\begin{equation}
    f(t)\cond h, x  = \int_{-\infty}^t h(t - \tau) x(\tau) \isd \tau
    \quad \text{where} \quad
    x \sim \GP(0, k_x(t, t')), \quad
    h \sim \GP(0, k_h(t, t')).
\end{equation}
Compared to the (C)GPCM, the filter $h$ has now a  white noise prior and the input signal $x$ is  given by an OU process.

\paragraph{Nonparametric kernel formulation:}
To derive the equivalent nonparametric kernel formulation, note that
\begin{equation}
    \E[f(t) \cond h]
    = \int_{-\infty}^t h(t - \tau) \E[x(\tau)] \isd \tau
    = 0
\end{equation}
and
\begin{align}
    \E[f(t)f(t') \cond h]
    &= \int_{-\infty}^t \int_{-\infty}^{t'} h(t - \tau)h(t' - \tau') \E[x(\tau)x(\tau')] \isd \tau' \isd \tau \\
    &= \int_{-\infty}^t \int_{-\infty}^{t'}  h(t - \tau)h(t' - \tau') k_x(\tau - \tau') \isd \tau' \isd \tau \\
    &= \int_{0}^\infty \int_{0}^{\infty} h(\tau)h(\tau') k_x((t - t') - (\tau - \tau')) \isd \tau' \isd \tau.
\end{align}
We conclude that the RGPCM is equivalent to the following generative model:
\begin{equation}\label{eq:rgpcm_LS}
    f \cond h \sim \GP\parens*{
        0, \int_{0}^\infty \int_{0}^{\infty} h(\tau)h(\tau') k_x((t - t') - (\tau - \tau')) \isd \tau' \isd \tau
    }, \quad
    h \sim \GP(0, k_h(t, t')).
\end{equation}

\paragraph{Nonparametric spectral formulation:}
Associated with the  the linear system \eqref{eq:rgpcm_LS} is the \emph{frequency response}  of the filter $h$
\begin{equation}\label{eq:FR_rgpcm}
    g(\omega) = \int_0^\infty e^{-i\omega t}h(t) \isd t.
\end{equation}
Note that $g$ is a Gaussian process, since the Fourier transform in \eqref{eq:FR_rgpcm} is a linear operator. From  the \emph{spectral representation} of $f$, it then follows that the spectral density of $f$ is given by
 \begin{equation}\label{eq:PSD_rgpcm}
     \phi_f(\omega) = |g(\omega)|^2\phi_x(\omega),
 \end{equation}
  see  \cite{Lindgren:2012:Stationary_Stochastic_Processes}.
 Effectively, the spectrum is a (random) modulation of the input signal's spectrum $\phi_x(\omega)=\frac{1}{\pi}\frac{\lambda}{\lambda^2+\omega^2}$. This can also be seen 
 through the Fourier duality with the nonparameteric kernel 
  \begin{equation}
     \E[f(t)f(t') \cond g] = \int e^{i\omega(t-t')}|g(\omega)|^2\phi_x(\omega)\isd\omega.
 \end{equation}
 Finally, we note that the spectral density of the fractional Ornstein--Uhlenbeck process with Hurst exponent $H$ is given by
 \begin{equation}
     \phi_{\text{fOU}}(\omega) = c|\omega|^{1-2H}\frac{\lambda}{\lambda^2 + \omega^2}
 \end{equation} where $c$ is a normalising constant, see \cite{cheridito2003fractionalOU}. Hence, this is a parametric special case of the nonparameteric spectrum \eqref{eq:PSD_rgpcm}, namely $|g(\omega)|^2 = c\pi|\omega|^{1-2H}$.

\subsection{Inducing Points}
\label{app:inducing_points_rgpcm}
For the filter $h$, since it now enjoys a white noise prior, we first define an inter-domain transformation, which we choose the be \emph{causal}:
\begin{equation}
    s(\tau) = \int_{-\infty}^\tau \tilde\gamma e^{-\gamma|\tau - t|} h(t) \isd t.
\end{equation}
As we will see below, by choosing the inter-domain transform to be causal, the kernel of the inter-domain process $s$ will take a simple form.
Then define $n_u$ inducing point inputs $\vt_u$ initialised to uniformly spaced over $[0, \tau_w]$ where $\tau_w$ is the extent of the filter (see \cref{sec:gpcms}) and let the inducing points $\vu$ be
\begin{equation}
    \vu = (s(t_{u,1}), \ldots, s(t_{u,n_u})).
\end{equation}
For the excitation signal $x$, we consider a collection of projections tailored for learning features in the spectral domain \citep{Hensman:2018:Variational_Fourier_Features_for_Gaussian}.
These features are defined on a window $[a, b]$ of interest.
Typically, this window should contain the locations of all observed data points and also points where you want to make predictions.
To define the features, let $M \in \N$ and consider the following basis functions:
\begin{equation} \label{eq:rgpcm-vffs}
    \beta_m(t) = \begin{cases}
        1 & \text{if $m = 0$,} \\
        \cos(\omega_m(t - a)) & \text{if $1 \le m \le M$}, \\
        \sin(\omega_{m - M}(t - a)) & \text{if $M < m \le 2M$},
    \end{cases}
\end{equation}
where the frequencies are harmonics on the interval $[a, b]$:
\begin{equation}
    \omega_m = \frac{2 \pi m}{b - a}, \quad m = 1, \ldots, M.
\end{equation}
Denote the concatenation of all these features by $\vbeta(t)$.
We then let the inducing points $\vz$ be
\begin{equation}
    \vz = (\lra{x, \beta_0}_\H, \ldots, \lra{x, \beta_{n_z}}_\H)
\end{equation}
where $n_z = 2M + 1$ and $\lra{\vardot,\vardot}_\H$ is the inner product corresponding to the reproducing kernel Hilbert space $\H$ associated with $k_x$.
With these choices, we have the covariances
\begin{align}
    k_\vu(t)
    &\coloneqq \E[h(t) s(\vt_u)]
    = \int_{-\infty}^{\vt_u} \tilde\gamma e^{-\gamma|\vt_u- t'|} \E[h(t)h(t')] \isd t'
    = \tilde\gamma e^{-\gamma(t - \vt_u)} \ind(t \le \vt_u), \\
    \mK_\vu &\coloneqq \E[s(\vt_u) s(\vt_u)^\T] = \frac{\tilde\gamma^2}{2\gamma} e^{-\gamma|\vt_u - \vt_u^\T|}, \\
    k_\vz(t)
    &\coloneqq \vphantom{\int_{-\infty}^{\vt^\T_u}} \E[x(t) \lra{x, \vbeta}_\H]
    = \lra{\E[x(t) x], \vbeta}_\H
    = \lra{k_x(t, \vardot), \vbeta}_\H
    = \vbeta(t), \\
    \mK_\vz
    &\coloneqq \vphantom{\int_{-\infty}^{\vt^\T_u}} \E[\lra{x, \vbeta}_\H \lra{x, \vbeta^\T}_\H]
    = \lra{\lra{\E[x(\vardot) x(\vardot)], \vbeta}_\H, \vbeta^\T}_\H
    = \lra{\lra{k_x(\vardot,\vardot), \vbeta}_\H, \vbeta^\T}_\H
    = \lra{\vbeta, \vbeta^\T}_\H,
\end{align}
where we use the reproducing property of $k_x$ on $\H$.
The expression for $\mK_\vu$ follows from
\begin{align}
    \E[s(\vt_u) s(\vt^\T_u)]
    &= \int_{-\infty}^{\vt_u} \int_{-\infty}^{\vt^\T_u} \tilde\gamma^2 e^{-\gamma|\vt_u- t| -\gamma|\vt_u^\T- t'|} \E[h(t)h(t')] \isd t' \isd t \\
    &= \int_{-\infty}^{\vt_u \land \vt^\T_u} \tilde\gamma^2 e^{-\gamma(\vt_u- t) -\gamma(\vt_u^\T- t)} \isd t \\
    &= \vphantom{\int_{-\infty}^{\vt^\T_u}} \frac{\tilde\gamma^2}{2\gamma} e^{-\gamma(\vt_u + \vt_u^\T) - 2 \gamma (\vt_u \land \vt^\T_u)} \\
    &= \vphantom{\int_{-\infty}^{\vt^\T_u}} \frac{\tilde\gamma^2}{2\gamma} e^{-\gamma|\vt_u - \vt_u^\T|}.
\end{align}
Note that $\mK_\vz$ requires explicit computation of $\lra{\beta_m, \beta_n}_\H$ for all $m, n = 0, \ldots, M$, which can be done using an explicit expression for $\lra{\vardot,\vardot}_\H$;
see \citet{Hensman:2018:Variational_Fourier_Features_for_Gaussian} for details.

\section{Initialisation of the GPCM, CGPCM, and RGPCM}
\label{app:init}

Let the subscript $\vardot\ss{ac}$ refer to the GPCM, $\vardot\ss{c}$ to the CGPCM, and $\vardot\ss{r}$ to the RGPCM.
In this section, we derive a comparable and fair initialisation for the three models.
This initialisation follows from the following two requirements:
(1) the prior marginal variance is unity and (2) the length scales of the prior mean covariance are equal.

The prior marginal variance of the models are as follows:
\begin{equation}
    P\ss{c} = \frac12 \tilde\alpha^2 \sqrt{\frac{\pi}{2 \alpha}},\quad
    P\ss{ac} = \tilde\alpha^2 \sqrt{\frac{\pi}{2 \alpha}}, \quad
    P\ss{r} = \frac{\tilde\alpha^2}{2 \alpha}
\end{equation}
where we note that $P\ss{c} = \tfrac12 P\ss{ac}$:
the causality constraint cuts the reach of the filter in half.
Define $\tilde\alpha$ by requiring that the power is one.
This gives
\begin{equation}
    \tilde\alpha\ss{c}^2 = 2\sqrt{\frac{2 \alpha}{\pi}}, \quad
    \tilde\alpha\ss{ac}^2 = \sqrt{\frac{2 \alpha}{\pi}}, \quad
    \tilde\alpha\ss{rc}^2 = 2 \alpha.
\end{equation}
With these choices for $\tilde\alpha$, we obtain the following prior mean covariance functions:
\begin{align}
    k\ss{ac}(r) &= \exp(-(\tfrac12\alpha + \gamma) r^2), \\
    k\ss{c}(r) &= (1 - \erf(\smash{\sqrt{\tfrac12 \alpha}}|r|)) \exp(-(\tfrac12\alpha + \gamma) r^2), \\
    k\ss{r}(r) &= \exp(-\lambda |r|) 
\end{align}
where we note that $k\ss{c}$ is a windowed version of $k\ss{ac}$.
This window $r \mapsto 1 - \erf(\smash{(\tfrac12 \alpha)^{\tfrac12}}|r|)$ is nondifferentiable at $r = 0$ and responsible for the nowhere differentiable sample paths of the CGPCM.

Define the length scale $\tau$ of a non-negative function $k \colon [0, \infty) \to \R$ by
\begin{equation}
    \tau = \frac{1}{k(0)}\int_0^\infty k(r) \isd r.
\end{equation}
Then the length scales of the windows $w\ss{ac}$, $w\ss{c}$, $w\ss{r}$ are given by
\begin{equation}
    \tau_{w,\mathrm{ac}} = \sqrt{\frac{\pi}{4\alpha}},\quad
    \tau_{w,\mathrm{c}} = \sqrt{\frac{\pi}{4\alpha}},\quad
    \tau_{w,\mathrm{v}} = \frac{1}{\alpha}
\end{equation}
and the length scales of the prior mean covariances are given by
\begin{align}
    \tau_{f,\mathrm{ac}}
    &= \sqrt{\frac{\pi}{2(\alpha + 2\gamma)}}, \\
    \tau_{f,\mathrm{c}}
    &= \sqrt{\frac{2}{\pi(\alpha + 2\gamma)}}\tanh\parens*{
        \sqrt{\frac{\alpha + 2 \gamma}{\alpha}}
    } \overset{\gamma \gg \alpha}{\approx} \sqrt{\frac{\pi}{2(\alpha + 2\gamma)}}, \\
    \tau_{f,\mathrm{r}} &= \frac{1}{\lambda}
\end{align}
where the approximation follows from that $\tanh(x) \approx \frac12 \pi$ for $x \gg 1$.
If we fix $\tau_w$ and $\tau_f$ to given values for all models, we obtain
\begin{equation}
    \alpha\ss{ac}
    = \frac{\pi}{4} \frac{1}{\tau^2_w},\quad
    \alpha\ss{c}
    = \frac{\pi}{4} \frac{1}{\tau^2_w},\quad
    \alpha\ss{r}
    = \frac{1}{\tau_w}.
\end{equation}
and
\begin{equation}
    \gamma\ss{ac,c}
    =\frac{\pi}{4} \frac{1}{\tau_f^2} - \frac12\alpha\ss{ac,c}, \quad
    \lambda\ss{r} = \frac{1}{\tau_f}.
\end{equation}
Intuitively, $\tau_f$ should be set to smallest length scale in the signal, and $\tau_w$ should be set to desired length of the filter, a bit larger than the largest length scale in the signal.

\section{Inference in the GPCM Family}
\label{app:inference}

\subsection{Implementation}
We provide a JAX \citep{jax2018github} based Python implementation of the GPCM, CGPCM, and RGPCM at \href{https://github.com/wesselb/gpcm}{\texttt{github.com/wesselb/gpcm}}.
The package provides an \texttt{sklearn}-style \texttt{model.fit(t, y)}--\texttt{model.predict(t\_new)} interface. 
The following is an example of fitting the RGPCM to data.
Here, \texttt{window} refers to $\tau_w$ and \texttt{scale} refers to $\tau_f$.

{\small\begin{minted}{python}
import numpy as np

from gpcm import GPCM, CGPCM, RGPCM 

model = RGPCM(window=2, scale=1, noise=0.1, t=(0, 10))

# Sample from the prior.
t = np.linspace(0, 10, 100)
k, y = model.sample(t)

# Fit model to the sample.
model.fit(t, y)

# Compute the ELBO.
elbo = model.elbo(t, y)

# Make predictions.
posterior = model.condition(t, y)
mean, var = posterior.predict(t)
\end{minted}
}

A key difficulty in the implementation is that the CGPCM requires evaluation and gradients of bivariate normal CDF; see \cref{subsec:integrals_gpcm_cgpcm}.

\subsection{Variational Approximation}

In what follows, $\vu$ are inducing points (or features) for the filter $h$ and $\vz$ are inducing points (or features) for the excitation signal $x$;
see \citet{Tobar:2015:Learning_Stationary} for the specification of $\vu$ and $\vz$ for the GPCM and \cref{app:inducing_points_cgpcm,app:inducing_points_rgpcm} for the specifications of $\vu$ and $\vz$ for respectively the CGPCM and RGPCM.
We consider the variational approximation 
\begin{equation}
    q_\th(h, x, \vu, \vz)
    = p_\th(h\cond \vu) p_\th(x \cond \vz) q(\vu, \vz)
\end{equation}
where $q(\vz, \vu)$ is a joint variational approximation over the inducing points.
Given some data $\vy$, to optimise $q(\vz, \vu)$ and $\theta$, we optimise the \emph{evidence lower bound} (ELBO):
\begin{equation}
    \F_\th[q(\vu, \vz)]
    = \E_{q}[\log p_\th(\vy\cond f)]
        - \operatorname{KL}[q(\vu, \vz)\divsep p_\th(\vu)p_\th(\vz)].
\end{equation}
As the name suggests, the ELBO indeed forms a lower bound on the marginal likelihood, $\F_\th[q(\vu, \vz)] \le \log p_\th(\vy)$, and optimising $\F_\th[q(\vu, \vz)]$ corresponds to minimising the Kullback--Leibler divergence between our approximation of the posterior $q_\th(h, x, \vu, \vz)$ and the true posterior $p_\th(h, x, \vu, \vz \cond \vy)$ \citep[see, \eg,][]{Wainwright:2008:Graphical_Models_Exponential_Families_and}.
Key to the tractability of the ELBO is the observation that
\begin{equation}
    \E_q[\log p_\theta(\vy \cond f)]
    = \E_{q(\vu, \vz)}[ \E[\log p_\theta(\vy \cond f)\cond \vu, \vz] ]
\end{equation}
where $\E[\log p_\theta(\vy \cond f)\cond \vu, \vz]$ happens to be \emph{conditionally quadratic} in $\vu$ and $\vz$;
see \cref{subsec:cond-exp-loglik}.
Throughout, we drop the dependency on $\th$ and denote $\hat\vu = \mK_\vu^{-1} \vu$ and $\hat\vz = \mK_\vz^{-1} \vz$.
In the remainder of this section, we consider two types of inference schemes: various mean-field schemes and a structured scheme.

\subsection{Mean-Field Inference}
In the mean-field inference scheme, the variational distribution is assumed to factorise
\begin{equation}
    q(\vu,\vz) = q(\vu)q(\vz)
\end{equation}
such that no dependency can occur between the variables. This leads to a simplifying factorisation of the ELBO:
\begin{equation}\label{eq:mf_elbo}
    \F_\th[q(\vu) q(\vz)]
    = \E_{q}[\log p_\th(\vy\cond f)]
        - \operatorname{KL}[q(\vu)\divsep p_\th(\vu)]- \operatorname{KL}[q(\vz)\divsep p_\th(\vz)].
\end{equation}
\cite{Tobar:2015:Learning_Stationary} assume Gaussians 
\begin{equation}
    q(\vu) = \mathcal{N}(\vu;\vmu_\vu,\mSigma_\vu), \quad q(\vz) = \mathcal{N}(\vz;\vmu_\vz,\mSigma_\vz)    
\end{equation}
and optimise \eqref{eq:mf_elbo} with respect the means $(\vmu_\vu,\vmu_\vz)$ and dense covariance matrices $(\mSigma_\vu,\mSigma_\vz)$ using gradient-based optimisation.
As \citeauthor{Tobar:2015:Learning_Stationary} remark and \cref{subsec:optimal_mean-field_qz} shows, the optimal mean-field approximations $q^*(\vu)$ and $q^*(\vz)$ are indeed of a Gaussian form.

\paragraph{Coordinate ascent:}
\Cref{subsec:optimal_mean-field_qz} not only shows that the optimal $q^*(\vu)$ and $q^*(\vz)$ are Gaussian, it also shows that, given $q(\vz)$ (resp.\ $q(\vu)$), the optimal $q^*(\vu)$  (resp.\ $q^*(\vz)$) can be computed explicitly.
This gives rise to a coordinate ascent (CA) scheme which avoids gradient-based optimisation:
\begin{enumerate}
    \item[(1)] Set $q_0(\vu) = p(\vu)$.
    \item[(2)] For $i = 1, \ldots, m$,
    \begin{enumerate}
        \item[(2.a)] set $q_i(\vz)$ to the optimal $q^*(\vz)$ given $q_{i-1}(\vu)$, and
        \item[(2.b)] set $q_i(\vu)$ to the optimal $q^*(\vu)$ given $q_{i}(\vz)$.
    \end{enumerate}
\end{enumerate}
The expression for the optimal $q^*(\vz)$ given $q(\vu)$ is computed in \eqref{eq:optimal_mean-field_qz-var} and \eqref{eq:optimal_mean-field_qz-mean} and the expression for the optimal $q^*(\vu)$ given $q(\vz)$ is exactly analogous.

\paragraph{Collapsed mean-field:}
In practice, the coordinate ascent scheme converges much quicker than gradient-based optimisation of the mean-field ELBO.
A downside of the coordinate ascent scheme, however, is that it only optimises the variational approximation: it cannot optimise the hyperparameters $\th$ and inducing point inputs.
For this, we propose to plug the optimal $q^*(\vz)$ given $q(\vu)$ back into the mean-field ELBO: $\F_\th[q(\vu) q^*(\vu)]$.
This \emph{collapsed} mean-field ELBO does not depend on the variational distribution $q(\vz)$ anymore: it depends on many fewer variational parameters, which greatly accelerates gradient-based optimisation.
The expression for the collapsed mean-field ELBO is computed in \eqref{eq:collapsed_mf_elbo}.

\paragraph{Computational complexity:}
The dominating computation of the three mean-field schemes is the computation of expressions of the form $\mI_{\vu\vz}(t_i) \mX \mI_{\vu\vz}^\T(t_i)$ for all $i = 1,\dots, n$ where $\mX$ is some $n_z \times n_z$ matrix (\eg, see \eqref{eq:dominating-computation}), which takes time $O(n (n_u n_z^2 + n_u^2 n_z))$.
Note that the analogous expression $\mI_{\vu\vz}^\T(t_i) \mY \mI_{\vu\vz}(t_i)$ for some $n_u \times n_u$ matrix $\mY$ is equally expensive.

\subsection{Structured Inference}
In the structured inference scheme, there is no independence assumption nor a Gaussianity assumption on the variational distribution: we consider a general potentially non-Gaussian $q(\vu, \vz)$ which can model arbitrary dependencies between $\vu$ and $\vz$.
Let
\begin{align}\label{eq:guz}
    g(\vu, \vz) &= \exp \E[\log p(\vy \cond f)\cond \vu, \vz], \\
    Z^* &= \int p(\vu) p(\vz) g(\vu, \vz) \isd \vu \isd \vz.
\end{align}
Then the ELBO can be written as
\begin{equation}
    \F_\th[q(\vu, \vz)]
    = \log Z^* - \KL(q(\vu, \vz) \divsep \tfrac1{Z^*} p(\vu) p(\vz) g(\vu, \vz)),
\end{equation}
which means that the optimal $q^*(\vu, \vz)$ is given by $q^*(\vu, \vz) = \tfrac1{Z^*} p(\vu) p(\vz) g(\vu, \vz)$.
As \cref{subsec:cond-exp-loglik} explains, $(\vu, \vz) \mapsto \E[\log p(\vy \cond f)\cond \vu, \vz]$ can be computed, but unfortunately is a \emph{quartic} function.
Therefore, although can we evaluate $g(\vu, \vz)$, we cannot compute $Z^*$ analytically, which means that we can only evaluate $q^*(\vu, \vz)$ up to a normalising constant.
Moreover, plugging $q^*(\vu, \vz)$ back into the ELBO gives
\begin{equation}
    \F_\th[q^*(\vu, \vz)] = \log Z^*,
\end{equation}
which is intractable because $Z^*$ cannot be computed.
In the next paragraphs, we describe how these intractabilities can be navigated to enable inference and learning with the structured approximation.

\paragraph{Inference through Gibbs sampling $q^*(\vu, \vz)$:}
The optimal $q^*(\vu, \vz)$ can be factorised as follows:
\begin{equation}\label{eq:facctorisation}
    q^*(\vu, \vz)
    = q^*(\vu) q^*(\vz \cond \vu) 
\end{equation}
where 
\begin{align}\label{eq:q_cond_optimal}
    q^*(\vu) &\vphantom{\int}\propto p(\vu) Z(\vu), \\
    q^*(\vz \cond \vu) &\vphantom{\int} = \frac{1}{Z(\vu)} p(\vz) g(\vz, \vu), \\
    Z(\vu) &= \int p(\vz) g(\vu, \vz) \isd \vz.
\end{align}
For $q^*(\vu)$, \cref{subsec:optimal_structured_qu} shows that it can only be evaluated up to a normalising constant, unfortunately.
However, crucially, \cref{subsec:optimal_structured_qzu} shows that $q^*(\vz \cond \vu)$ takes the form of a Gaussian with mean and variance depending on $\vu$, which means that, for a given $\vu$, $q^*(\vz \cond \vu)$ can be directly sampled from.
Therefore, although we cannot evaluate $q^*(\vu, \vz)$ nor directly sample from it, we can eventually generate samples from $q^*(\vu, \vz)$ with the following Gibbs sampling scheme:
\begin{enumerate}
    \item[(1)] Draw initial sample $\vu^{(0)} \sim p(\vu)$.
    \item[(2)]
        For $i = 1, \ldots, m$,
        \begin{enumerate}
            \item[(2.a)] sample $\vz^{(i)} \sim q^*(\vz \cond \vu^{(i-1)})$, and
            \item[(2.b)] sample $\vu^{(i)} \sim q^*(\vu \cond \vz^{(i)})$. 
        \end{enumerate}
\end{enumerate}
With this sampling scheme, we are able to perform inference in the model whilst completely avoiding numerical variational optimisation.
The expression for $q^*(\vz \cond \vu)$ is computed in \eqref{eq:optimal_structured_qzu_var} and \eqref{eq:optimal_structured_qzu_mean} and the expression for $q^*(\vu \cond \vz)$ is exactly analogous.

\paragraph{Learning through approximation of $\frac{\sd}{\sd\theta}
    \F_\th[q^*(\vu, \vz)]$:}
If we substitute the factorisation \eqref{eq:facctorisation} back into the ELBO, we get
\begin{equation}
    \F_\th[q^*(\vu, \vz)]
    = \F_\th[q^*(\vz \cond \vu) q^*(\vu)]
    = \E_{q^*(\vu)}[\log Z(\vu)] - \KL(q^*(\vu)\divsep p(\vu)).
\end{equation}
As before, $\F_\th[q^*(\vz \cond \vu) q^*(\vu)]$ is intractable.
However, it turns out that gradients of $\F_\th[q^*(\vz \cond \vu) q^*(\vu)]$ with respect to $\th$ and inducing points inputs can be computed:
\begin{equation}
    \frac{\sd}{\sd\theta}
    \F_\th[q^*(\vu, \vz)]
    =
        \frac{\d}{\d \th} F_\th[q^*(\vu, \vz)]
        + \cancel{\lra*{
            \frac{\delta F_\th}{\delta q(\vu, \vz)}[q^*(\vu, \vz)] ,
            \frac{\d}{\d \th} q^*(\vu, \vz)
        }}
\end{equation}
where $\smash{\frac{\sd}{\sd\theta}}$ denotes the total derivative with respect to $\th$, $\smash{\frac{\d}{\d \th}}$ the partial derivative with respect to $\th$, and $\smash{\frac{\delta }{\delta q(\vu, \vz)}}$ the functional derivative with respect to $q(\vu, \vz)$.
Here the second term cancels because $\smash{\frac{\delta F_\th}{\delta q(\vu, \vz)}[q^*(\vu, \vz)] = 0}$ by optimality of $q^*(\vu, \vz)$.
Therefore, using the above Gibbs sampling scheme to generate samples from $q^*(\vu)$, the following approximation is tractable:
\begin{equation} \label{eq:gradient-approx}
    \frac{\sd}{\sd\theta}
    \F_\th[q^*(\vu, \vz)]
    \approx \frac{\d}{\d\theta} \frac{1}{m} \sum_{i=1}^m \log Z_\th(\vu^{(i)}) p_\th(\vu^{(i)})
    \quad\text{where}\quad
    \vu^{(i)} \overset{\text{i.i.d.}}{\sim} q^*(\vu)
\end{equation}
where we make the dependence of $Z_\th(\vu)$ and $p_\th(\vu)$ on $\th$ explicit.
By iterating the Gibbs sampling scheme and \eqref{eq:gradient-approx}, we are able to perform stochastic gradient-based optimisation of $\th$ and the inducing point inputs.
The expression for $\log Z(\vu) p(\vu)$ is computed in \eqref{eq:gradient-target}.

\paragraph{Evidence approximation through a lower bound on $\F_\th[q^*(\vu, \vz)]$:}
Although we have described how gradients of $\F_\th[q^*(\vu, \vz)]$ can be approximated, some applications may require an approximation of the value of $\F_\th[q^*(\vu, \vz)]$.
To tractably approximate $\F_\th[q^*(\vu, \vz)]$, we propose the following lower bound:
\begin{equation}
    \F_\th[q^*(\vu, \vz)]
    \ge \F_\th[q^*\ss{MF}(\vu) q^*(\vz \cond \vu)]
    = \E_{q^*\ss{MF}(\vu)}[\log Z(\vu)] - \KL(q^*\ss{MF}(\vu)\divsep p(\vu))
\end{equation}
where $q^*\ss{MF}(\vu)$ is the optimal mean-field approximation of $q(\vu)$ obtained with the coordinate ascent procedure.
The expectation can be approximated with a Monte Carlo approximation by sampling from $q^*\ss{MF}(\vu)$.
The expression for $\F_\th[q(\vu) q^*(\vz \cond \vu)]$ for an arbitrary $q(\vu)$ is computed in \eqref{eq:partial_structured_bound}.

\paragraph{Computational complexity:}
Like the mean-field schemes, the structured inference scheme also takes $O(n(n_u n_z^2 + n_u^2 n_z))$ time.
However, for the Gibbs sampling scheme, by paying an upfront cost of $O(n(n_u n_z^2 + n_u^2 n_z))$ once, every Gibbs sample iteration can be performed in $O(n n_u n_z)$ time, which is a dramatic speed-up over $O(n(n_u n_z^2 + n_u^2 n_z))$. 
This speed-up makes it possible to always run the Gibbs sampler until convergence without excessive computational expense.

\subsection{Computations}
In this section, we give detailed derivations of all remaining computations.

\subsubsection{Conditional Expectation of the Likelihood \texorpdfstring{$\E[\log p(\vy \cond f) \cond \vu, \vz]$}{E[log p(y | f) | u, z]}}
\label{subsec:cond-exp-loglik}
To begin with, compute
\begin{equation}
    \E[f(t) \cond \vu, \vz]
    = \int_{-\infty}^t \E[h(t - \tau) \cond \vu] \E[x(\tau) \cond \vz] \isd \tau
    = \hat\vu^\T \int_{-\infty}^t
        k_{\vu}(t - \tau)
        k_{\vz}^\T(\tau)
    \isd \tau\, \hat\vz
    = \hat\vu^\T \mI_{\vu\vz}(t) \hat\vz
\end{equation}
where we define
\begin{equation}
    \mI_{\vu\vz}(t) \coloneqq \int_{-\infty}^t
        k_{\vu}(t - \tau)
        k_{\vz}^\T(\tau)
    \isd \tau.
\end{equation}
Write
\begin{equation}
    \E[h(t) h(t') \cond \vu]
    =
        k_h(t, t')
        -
        k_{\vu}^\T(t) \mK_\vu^{-1} k_{\vu}(t')
        + (k_{\vu}^\T(t) \hat\vu)^2
    =
        k_h(t, t')
        +
        k_{\vu}^\T(t) \mM_\vu k_{\vu}(t')
\end{equation}
where $\mM_\vu = \hat\vu\hat\vu^\T - \mK_\vu^{-1}$.
Then
\begin{align}
    \E[f^2(t) \cond \vu, \vz]
    &= \int_{-\infty}^t \int_{-\infty}^{t'} \E[h(t - \tau) h(t - \tau') \cond \vu] \E[x(\tau) x(\tau') \cond \vz] \isd \tau' \isd \tau  \\
    &= \int_{-\infty}^t \int_{-\infty}^{t'} (
            k_h(t - \tau, t - \tau')
            +
            k_{\vu}^\T(t - \tau) \mM_\vu k_{\vu}(t - \tau')
        ) (
            k_x(\tau, \tau')
            +
            k_{\vz}^\T(\tau) \mM_\vz k_{\vz}(\tau')
        ) \isd \tau' \isd \tau\\
    &= \vphantom{\int} T_1(t) + T_2(t) + T_3(t) + T_4(t)
\end{align}
where
\begin{align}
    T_1(t)
    &\coloneqq \int_{-\infty}^t \int_{-\infty}^{t} k_h(t - \tau, t - \tau') k_x(\tau, \tau') \isd \tau' \isd \tau \\
    T_2(t)
    &\coloneqq \int_{-\infty}^t \int_{-\infty}^{t} k_h(t - \tau, t - \tau') k_{\vz}^\T(\tau) \mM_\vz k_{\vz}(\tau')  \isd \tau' \isd \tau \\
    &= \tr \mM_\vz \int_{-\infty}^t \int_{-\infty}^{t} k_h(t - \tau, t - \tau')   k_{\vz}(\tau') k_{\vz}^\T(\tau) \isd \tau' \isd \tau \\
    T_3(t)
    &\coloneqq
        \int_{-\infty}^t \int_{-\infty}^{t}
            k_{\vu}^\T(t - \tau) \mM_\vu k_{\vu}(t - \tau')
            k_x(\tau, \tau')
        \isd \tau' \isd \tau \\
    &=
        \tr \mM_\vu \int_{-\infty}^t \int_{-\infty}^{t}
            k_{\vu}(t - \tau') k_{\vu}^\T(t - \tau)
            k_x(\tau, \tau')
        \isd \tau' \isd \tau \\
    T_4(t)
    &\coloneqq
        \int_{-\infty}^t \int_{-\infty}^{t}
            k_{\vu}^\T(t - \tau) \mM_\vu k_{\vu}(t - \tau')
            k_{\vz}^\T(\tau) \mM_\vz k_{\vz}(\tau')
        \isd \tau' \isd \tau \\
    &=
        \tr
            \mM_\vz
            \int_{-\infty}^t
                k_{\vz}(\tau)
                k_{\vu}^\T(t - \tau)
            \isd \tau\,
            \mM_\vu
            \int_{-\infty}^{t}
                k_{\vu}(t - \tau')
                k_{\vz}^\T(\tau')
            \isd \tau' \\
    &=
        \tr
            \mM_\vz
            \mI_{\vu\vz}^\T(t)
            \mM_\vu
            \mI_{\vu\vz}(t).
\end{align}
Further define
\begin{align}
    I_{hx}(t)
    &\coloneqq \int_{-\infty}^t \int_{-\infty}^{t} k_h(t - \tau, t - \tau') k_x(\tau, \tau') \isd \tau' \isd \tau, \\
    \mI_{h\vz}(t)
    &\coloneqq \int_{-\infty}^t \int_{-\infty}^{t} k_h(t - \tau, t - \tau')
        k_{\vz}(\tau') k_{\vz}^\T(\tau)
    \isd \tau' \isd \tau, \\
    \mI_{\vu x}(t)
    &\coloneqq \int_{-\infty}^t \int_{-\infty}^t
        k_{\vu}(t - \tau') k_{\vu}^\T(t - \tau) k_x(\tau, \tau') \isd \tau'
    \isd \tau.
\end{align}
Then
\begin{align}
    T_1(t) &= I_{hx}(t),  \\
    T_2(t)
        &= \lra{\mM_\vz, \mI_{h\vz}(t)} \\
        &= \hat \vz^\T \mI_{h\vz}(t) \hat \vz - \lra{\mK_\vz^{-1}, \mI_{h\vz}(t)}, \\
    T_3(t)
        &= \lra{\mM_\vu, \mI_{\vu x}(t)} \\
        &= \hat \vu^\T \mI_{\vu x}(t) \hat \vu - \lra{\mK_{\vu}^{-1}, \mI_{\vu x}(t)}, \\
    T_4(t)
        &= \lra{\mM_\vu, \mI_{\vu\vz}(t)\mM_\vz \mI_{\vu\vz}^\T(t)}, \\
        &= \lra{\hat\vu\hat\vu^\T - \mK_\vu^{-1}, \mI_{\vu\vz}(t)(\hat\vz\hat\vz^\T - \mK_\vz^{-1}) \mI_{\vu\vz}^\T(t)}, \\
        &=
            (\hat \vu^\T \mI_{\vu\vz}(t) \hat \vz)^2
            - \hat \vu^\T \mI_{\vu\vz}(t)\mK_\vz^{-1}\mI_{\vu\vz}^\T(t)\hat \vu  \nonumber \\
        &\qquad
            - \hat \vz^\T \mI_{\vu\vz}^\T(t)\mK_\vu^{-1}\mI_{\vu\vz}(t)\hat \vz
            + \lra{\mK^{-1}_\vu, \mI_{\vu\vz}(t)\mK^{-1}_\vz \mI_{\vu\vz}^\T(t)}.
\end{align}
Aggregate
\begin{align}
    \mA(t) &\coloneqq
        \mI_{\vu x}(t) - \mI_{\vu\vz}(t)\mK_\vz^{-1}\mI_{\vu\vz}^\T(t), \\
    \mB(t) &\coloneqq
        \mI_{h\vz}(t) - \mI_{\vu\vz}^\T(t)\mK_\vu^{-1}\mI_{\vu\vz}(t), \\
    c(t) &\coloneqq
        I_{hx}(t)
        - \lra{\mK_{\vu}^{-1}, \mI_{\vu x}(t)}
        - \lra{\mK_\vz^{-1}, \mI_{h\vz}(t)}
        + \lra{\mK^{-1}_\vu, \mI_{\vu\vz}(t)\mK^{-1}_\vz \mI_{\vu\vz}^\T(t)}. \label{eq:dominating-computation}
\end{align}
Then
\begin{equation}
    \E[f^2(t) \cond \vu, \vz]
    = \hat \vu^\T \mA(t) \hat \vu + \hat \vz^\T \mB(t) \hat \vz + (\hat \vu^\T \mI_{\vu\vz}(t) \hat \vz)^2 + c(t).
\end{equation}
With $\E[f(t)\cond \vu, \vz]$ and $\E[f^2(t)\cond \vu, \vz]$ computed, we can compute $\E[\log p(\vy \cond f) \cond \vu, \vz]$:
\begin{align}
    \E[\log p(\vy \cond f) \cond \vu, \vz]
    &= \sum_{i=1}^n \sbrac*{
        -\frac12\log(2 \pi \sigma^2)
        - \frac1{2 \sigma^2}\parens*{
            \E[f^2(t_i)\cond \vu, \vz] - 2 y_i \E[f(t_i) \cond \vu, \vz]
            + y_i^2
        }
    } \\
    &\;=
    -\frac n2\log(2 \pi \sigma^2)  - \frac1{2\sigma^2}\norm{\vy}^2_2\nonumber \\
    &\qquad - \frac1{2 \sigma^2}\sum_{i=1}^n \sbrac*{
        \hat \vu^\T \mA(t_i) \hat \vu + \hat \vz^\T \mB(t_i) \hat \vz + (\hat \vu^\T \mI_{\vu\vz}(t_i) \hat \vz)^2 + c(t_i)
        - 2 y_i \hat\vu^\T \mI_{\vu\vz}(t_i) \hat\vz
    }. \\
    &\;=
    -\frac n2\log(2 \pi \sigma^2) - \frac1{2\sigma^2}\sbrac*{\norm{\vy}^2_2 + \hat\vu^\T \sum_{i=1}^n \mA(t_i) \hat \vu + \sum_{i=1}^nc(t_i)} \nonumber\\
    &\qquad - \frac1{2 \sigma^2}\sum_{i=1}^n \sbrac*{
        \hat \vz^\T \parens*{
            \mB(t_i)
            + \mI_{\vu\vz}^\T(t_i) \hat \vu \hat \vu^\T \mI_{\vu\vz}(t_i)
        }\hat \vz
        - 2 y_i \hat\vu^\T \mI_{\vu\vz}(t_i) \hat\vz
    }.
\end{align}
Crucially, observe that both $\vu \mapsto \E[\log p(\vy \cond f) \cond \vu, \vz]$ and $\vz \mapsto \E[\log p(\vy \cond f) \cond \vu, \vz]$ are quadratic functions.
We therefore say that $\E[\log p(\vy \cond f) \cond \vu, \vz]$ is \emph{conditionally quadratic} in $\vu$ and $\vz$.
The function $(\vu, \vz) \mapsto \E[\log p(\vy \cond f) \cond \vu, \vz]$, however, is quartic rather than quadratic.
It remains to compute the integrals $I_{hx}(t)$, $\mI_{h\vz}(t)$, $\mI_{\vu x}(t)$, and $\mI_{\vu\vz}(t)$.
We will do this for the GPCM
and CGPCM in \cref{subsec:integrals_gpcm_cgpcm} and for the RGPCM in \cref{subsec:integrals_rgpcm}.

\subsubsection{Optimal Mean-Field \texorpdfstring{$q^*(\vz)$}{q(z)} Given \texorpdfstring{$q(\vu)$}{q(u)} and the Collapsed Mean-Field ELBO}
\label{subsec:optimal_mean-field_qz}
We borrow the result from the next section, which computes the optimal $q^*(\hat \vz \cond \hat \vu)$ and the partial structured ELBO.
To instead compute the optimal $q^*(\vz)$ given $q(\vu)$ in the mean-field scheme and the collapsed mean-field ELBO, simply take an additional expectation over $q(\vu)$.
In particular, let $q(\hat\vu) = \Normal(\vmu_{\hat\vu}, \mSigma_{\hat\vu})$.
Then $q^*(\hat \vz) = \Normal(\vmu_{\hat\vz},\mSigma_{\hat\vz})$ where
\begin{align} 
    \mSigma_{\hat\vz}^{-1}
    &= \mK_\vz + \frac1{\sigma^2} \sbrac*{
        \sum_{i=1}^n \mB(t_i)
        + \sum_{i=1}^n \mI_{\vu\vz}^\T(t_i) (\mSigma_{\hat\vu} + \vmu_{\hat\vu}\vmu_{\hat\vu}^\T) \mI_{\vu\vz}(t_i)
    }, \label{eq:optimal_mean-field_qz-var} \\
    \mSigma_{\hat\vz}^{-1} \vmu_{\hat\vz}
    &= \frac1{\sigma^2} \sum_{i=1}^n y_i \mI_{\vu\vz}^\T(t_i) \vmu_{\hat \vu}. \label{eq:optimal_mean-field_qz-mean} 
\end{align}
Similarly, 
\begin{align}
    &\log \int p(\vz) \exp \E_{q(\vu)}[\E[\log p(\vy \cond f) \cond \vu, \vz]] \isd \vz \nonumber \\
    &\qquad=
    -\frac n2\log(2 \pi \sigma^2)- \frac1{2\sigma^2}\sbrac*{\norm{\vy}^2_2 + \hat\vu^\T \sum_{i=1}^n \mA(t_i) \hat \vu + \sum_{i=1}^n c(t_i)}
        + \frac12 \log \,\abs{\mSigma_{\hat\vz}}
        + \frac12 \log \,\abs{\mK_\vz}
        + \frac12 \vmu_{\hat\vz}^\T \mSigma_{\hat\vz}^{-1} \vmu_{\hat\vz},
\end{align}
so
\begin{align}
    &\F_\theta[q^*(\vz)q(\vu)] \nonumber \\
    &\qquad=
    -\frac n2\log(2 \pi \sigma^2)- \frac1{2\sigma^2}\sbrac*{\norm{\vy}^2_2 + \tr\sbrac*{
        (\mSigma_{\hat\vu} + \vmu_{\hat\vu} \vmu_{\hat\vu}^\T) \sum_{i=1}^n \mA(t_i)
    } + \sum_{i=1}^n c(t_i)} \nonumber \\
    &\qquad\qquad
        + \frac12 \log \,\abs{\mSigma_{\hat\vz}}
        + \frac12 \log \,\abs{\mK_\vz}
        + \frac12 \vmu_{\hat\vz}^\T \mSigma_{\hat\vz}^{-1} \vmu_{\hat\vz}
        - \KL(q(\vu)\divsep p(\vu)). \label{eq:collapsed_mf_elbo}
\end{align}

\subsubsection{Optimal Structured \texorpdfstring{$q^*(\vz \cond \vu)$}{q(z | u)} and the Partially Structured ELBO}
\label{subsec:optimal_structured_qzu}

To begin with, expand
\begin{align}
    &\log p(\vz) + \E[\log p(\vy \cond f) \cond \vu, \vz] \\
    &\quad=
    -\frac n2\log(2 \pi \sigma^2) -\frac12\log\abs{2 \pi \mK_\vz} - \frac12\hat\vz^\T \mK_\vz \hat\vz - \frac1{2\sigma^2}\sbrac*{\norm{\vy}^2_2 + \hat\vu^\T \sum_{i=1}^n \mA(t_i) \hat \vu + \sum_{i=1}^n c(t_i)} \nonumber\\
    &\quad\qquad - \frac1{2 \sigma^2}\sum_{i=1}^n \sbrac*{
        \hat \vz^\T \parens*{
            \mB(t_i)
            + \mI_{\vu\vz}^\T(t_i) \hat \vu \hat \vu^\T \mI_{\vu\vz}(t_i)
        }\hat \vz
        - 2 y_i \hat\vu^\T \mI_{\vu\vz}(t_i) \hat\vz
    }.
\end{align}
Note that
\begin{align}
    \log \Normal(\vx;\vmu, \mSigma)
    &= -\frac12 \log\abs{2 \pi \mSigma} - \frac12 (\vx - \vmu)^\T \mSigma^{-1} (\vx - \vmu) \\
    &= -\frac12 \log\abs{2 \pi \mSigma} - \frac12 (
    \vx^\T \mSigma^{-1} \vx - 2 \vmu^\T \mSigma^{-1} \vx
    ) - \frac12 \vmu^\T \mSigma^{-1} \vmu,
\end{align}
so
\begin{equation}
    - \frac12 (
    \vx^\T \mSigma^{-1} \vx - 2 \vmu^\T \mSigma^{-1} \vx
    )
    =
        \log \Normal(\vx;\vmu, \mSigma)
        + \frac12 \log\abs{2 \pi \mSigma}
        + \frac12 \vmu^\T \mSigma^{-1} \vmu.
\end{equation}
Hence, observe that $q^*(\hat \vz \cond \hat \vu) = \Normal(\vmu_{\hat\vz},\mSigma_{\hat\vz}^{-1})$ where
\begin{align} 
    \mSigma_{\hat\vz}^{-1}
    &= \mK_\vz + \frac1{\sigma^2} \sbrac*{
        \sum_{i=1}^n \mB(t_i)
        + \sum_{i=1}^n \mI_{\vu\vz}^\T(t_i) \hat \vu \hat \vu^\T \mI_{\vu\vz}(t_i)
    }, \label{eq:optimal_structured_qzu_var} \\
    \mSigma_{\hat\vz}^{-1} \vmu_{\hat\vz}
    &= \frac1{\sigma^2} \sum_{i=1}^n y_i \mI_{\vu\vz}^\T(t_i) \hat \vu. \label{eq:optimal_structured_qzu_mean} 
\end{align}
Thus
\begin{align}
    &\log p(\vz) + \E[\log p(\vy \cond f) \cond \vu, \vz] \nonumber \\
    &\quad=
    -\frac n2\log(2 \pi \sigma^2) -\frac12\log\abs{2 \pi \mK_\vz} - \frac1{2\sigma^2}\sbrac*{\norm{\vy}^2_2 + \hat\vu^\T \sum_{i=1}^n \mA(t_i) \hat \vu + \sum_{i=1}^n c(t_i)} \nonumber\\
    &\quad\qquad
        + \log \Normal(\hat \vz; \vmu_{\hat\vz}, \mSigma_{\hat\vz})
        + \frac12 \log \,\abs{2 \pi \mSigma_{\hat\vz}}
        + \frac12 \vmu_{\hat\vz}^\T \mSigma_{\hat\vz}^{-1} \vmu_{\hat\vz} \\
    &\quad=
    -\frac n2\log(2 \pi \sigma^2)- \frac1{2\sigma^2}\sbrac*{\norm{\vy}^2_2 + \hat\vu^\T \sum_{i=1}^n \mA(t_i) \hat \vu + \sum_{i=1}^n c(t_i)} \nonumber\\
    &\quad\qquad
        + \log \Normal(\hat \vz; \vmu_{\hat\vz}, \mSigma_{\hat\vz})
        + \frac12 \log \,\abs{\mSigma_{\hat\vz}}
        - \frac12 \log \,\abs{\mK_\vz}
        + \frac12 \vmu_{\hat\vz}^\T \mSigma_{\hat\vz}^{-1} \vmu_{\hat\vz}.
\end{align}
We now apply $\exp$, integrate over $\sd \vz = \abs{\mK_\vz} \isd \hat \vz$, and apply $\log$ to find
\begin{align}
    &\log \int p(\vz) \exp \E[\log p(\vy \cond f) \cond \vu, \vz] \isd \vz \nonumber \\
    &\qquad= \log Z(\vu) \\
    &\qquad=
    -\frac n2\log(2 \pi \sigma^2)- \frac1{2\sigma^2}\sbrac*{\norm{\vy}^2_2 + \hat\vu^\T \sum_{i=1}^n \mA(t_i) \hat \vu + \sum_{i=1}^n c(t_i)}
        + \frac12 \log \,\abs{\mSigma_{\hat\vz}}
        + \frac12 \log \,\abs{\mK_\vz}
        + \frac12 \vmu_{\hat\vz}^\T \mSigma_{\hat\vz}^{-1} \vmu_{\hat\vz},
\end{align}
so
\begin{align}
    &\F_\th[q(\vu)q^*(\vu \cond \vz)] \nonumber \\
    &\qquad=
        -\frac n2\log(2 \pi \sigma^2)- \frac1{2\sigma^2}\sbrac*{\norm{\vy}^2_2 + \sum_{i=1}^n c(t_i)}- \frac12 \log \,\abs{\mK_\vz} \nonumber \\
    &\qquad\qquad
      + \E_{q(\vu)}\sbrac*{
        -\frac1{2\sigma^2}\hat\vu^\T \sum_{i=1}^n \mA(t_i) \hat \vu
        + \frac12 \log \,\abs{\mSigma_{\hat\vz}(\hat\vu)}
        + \frac12 \vmu_{\hat\vz}^\T(\hat\vu) \mSigma_{\hat\vz}^{-1}(\hat\vu) \vmu_{\hat\vz}(\hat\vu)
    } - \KL(q(\vu), p(\vu)). \label{eq:partial_structured_bound}
\end{align}
Here the expectation can be approximated using Monte Carlo.

\subsubsection{Optimal Structured \texorpdfstring{$q^*(\vu)$}{q(u)} and Gradients for the Structured ELBO}
\label{subsec:optimal_structured_qu}

Expand
\begin{align}
    \log q^*(\vu)
    &\simeq \log p(\vu) Z(\vu) \nonumber \\
    &=
        - \frac n2\log(2 \pi \sigma^2)
        - \frac1{2\sigma^2}\sbrac*{\norm{\vy}^2_2 + \hat\vu^\T \sum_{i=1}^n \mA(t_i) \hat \vu + \sum_{i=1}^n c(t_i)} \nonumber\\
    &\qquad
        - \frac12  \log \,\abs{2\pi \mK_\vu}
        - \frac12 \hat\vu^\T \mK_\vu \hat\vu
        + \frac12 \log \,\abs{\mSigma_{\hat\vz}(\hat\vu)}
        - \frac12 \log \,\abs{\mK_\vz}
        + \frac12 \vmu_{\hat\vz}^\T(\hat\vu) \mSigma_{\hat\vz}^{-1}(\hat\vu) \vmu_{\hat\vz}(\hat\vu)
        \label{eq:gradient-target}
\end{align}
where make the dependence of $\vmu_{\hat\vz}(\hat\vu)$ and $\mSigma_{\hat\vz}(\hat\vu)$ on $\hat\vu$ explicit.
We see that $q^*(\hat\vu)$ can be evaluated up to a normalising constant.

\subsection{Integrals for the GPCM and CGPCM}
\label{subsec:integrals_gpcm_cgpcm}
Rather than explicitly computing the integrals $I_{hx}(t)$, $\mI_{h\vz}(t)$, $\mI_{\vu x}(t)$, and $\mI_{\vu\vz}(t)$, we note that, for the GPCM and CGPCM, all these constitute integrals of exponentiated quadratic forms.
We therefore implement a small computer algebra system (CAS) which is able to symbolically solve the integrals and implement the solutions in JAX.
(For the GPCM, the integrals $I_{hx}(t)$, $\mI_{h\vz}(t)$, $\mI_{\vu x}(t)$, and $\mI_{\vu\vz}(t)$ are defined with different limits, but the CAS is able to handle that.)
For the CGPCM, this requires availability of the bivariate normal CDF.
For this, we use the FORTRAN implementation \texttt{TVPACK} by Alan Genz\footnote{\url{http://www.math.wsu.edu/faculty/genz/software/software.html}}, parallelise the implementation in C++ using OpenMP, and hook the result into JAX's JIT compiler with manually defined gradients.

Below is an example of using the CAS to compute the integral $\mI_{h\vz}(t)$:
{\small
\begin{minted}{python}
import numpy as np

from gpcm.exppoly import ExpPoly, const, var

t = np.linspace(0, 10, 100)
t_z = np.linspace(0, 10, 10)

alpha = 1
alpha_t = 1
gamma = 2
omega = 1
omega_t = 1


def k_h(t1, t2):
    return alpha_t ** 2 * ExpPoly(
        -(const(alpha) * (t1 ** 2 + t2 ** 2) + const(gamma) * (t1 - t2) ** 2),
    )


def k_xs(t1, t2):
    return omega_t * ExpPoly(-const(omega) * (t1 - t2) ** 2)


expq = (
    k_h(var("t") - var("tau1"), var("t") - var("tau2"))
    * k_xs(var("tau1"), var("t_z_1"))
    * k_xs(var("t_z_2"), var("tau2"))
)

I_hz = expq.integrate_box(
    ("tau1", -np.inf, var("t")),
    ("tau2", -np.inf, var("t")),
    t=t[:, None, None],
    t_z_1=t_z[None, :, None],
    t_z_2=t_z[None, None, :],
)
\end{minted}
}

\subsection{Integrals for the RGPCM}
\label{subsec:integrals_rgpcm}

\newcommand{\alphat}{\tilde\alpha}
\newcommand{\gammat}{\tilde\gamma}
\newcommand{\lambdat}{\tilde\lambda}

In what follows, recall that
\begin{equation}
    k_h(t, t') = w(t) w(t') k_g(t, t')
    \quad\text{with}\quad
    w(t) = \alphat e^{-\alpha |t|}, \quad
    k_g(t, t') = \delta(t - t')
\end{equation}
and
\begin{equation}
    k_x(t, t') = e^{-\lambda|t - t'|}.
\end{equation}

\subsubsection{Integral \texorpdfstring{$I_{hx}(t)$}{Ihx}}
Compute
\begin{align}
	I_{hx}(t,t')
	& = \int^{t}_{-\infty}  \int^{t'}_{-\infty} k_h(t-\tau,t'-\tau') k_x(\tau,\tau') \isd\tau'\isd\tau \\
	& = \int^{t}_{-\infty}  \int^{t'}_{-\infty} w(t-\tau)w(t'-\tau')k_g(t-\tau,t'-\tau') k_x(\tau,\tau') \isd\tau'\isd\tau \\
	& = \int_0^\infty  w^2(\tau) k_x(t - \tau,t' - \tau) \isd\tau \\
	& = \frac{\alphat^2}{2\alpha}e^{-\lambda|t - t'|},
\end{align}
so
\begin{equation}
	I_{hx} = I_{hx}(t,t) = \frac{\alphat^2}{2\alpha}.
\end{equation}

\subsubsection{Integral
    \texorpdfstring{$\mI_{h\vz}(t)$}{Ihz}
}\label{apdx::ihz}
Denote the $(m,n)$\textsuperscript{th} element of $\mI_{h\vz}(t)$ by
\begin{align}
	[I_{h\vz}(t)]_{m,n}
	&= \int^{t}_{-\infty} \int^{t}_{-\infty}  w(t-\tau)w(t-\tau') k_g(t-\tau,t-\tau')k_{z_m}(\tau) k_{z_n}(\tau') \isd\tau' \isd\tau \\
	&= \int^{t}_{-\infty} w^2(t-\tau) k_{z_m}(\tau) k_{z_n}(\tau') \isd\tau \\
	&\eqqcolon \vphantom{\int^{t}_{-\infty}} I_{m,n}(t).
\end{align}
For $m=n=0$, we have
\begin{align}
	I_{0,0}(t)
	&= \int_a^{t} w^2(t-\tau)\isd\tau + \int_{-\infty}^a w^2(t-\tau) e^{-2\lambda(a-\tau)}\isd\tau \\
	&= \frac{\alphat^2}{2\alpha}(1-e^{-2\alpha(t-a)}) + \frac{\alphat^2}{2(\alpha+\lambda)}e^{-2\alpha(t-a)} \\
	&= \frac{\alphat^2}{2\alpha} - \frac{\lambda \alphat^2}{2 \alpha (\alpha+\lambda)}e^{-2\alpha(t-a)}.
\end{align}
For $m=0$ and $1\le n\leq M$, we have cosine features:
\begin{align}
	I_{0,n:\cos}(t)
	&= \int_a^{t} w^2(t-\tau) \cos(\omega_n(\tau-a))\isd\tau + \int_{-\infty}^a w^2(t-\tau) e^{-2\lambda(a-\tau)}\isd\tau \\
	&= \frac{\alphat^2}{4\alpha^2+\omega_n^2}\left[  2\alpha\left( \cos(\omega_n(t-a)) - e^{-2\alpha(t-a)}\right) + \omega_n\sin(\omega_n(t-a)) \right]+ \frac{\alphat^2}{2(\alpha+\lambda)}e^{-2\alpha(t-a)} \\
	&= I_{0,0}(t) \vphantom{\int_a^{t}} \quad \text{if} \quad \omega_n=0.
\end{align}
Similarly, for $m,n\leq M$,
\begin{align}
	I_{m:\cos,n:\cos}(t)
	&=
		\int_a^{t}\!\! w^2(t-\tau) \cos(\omega_m(\tau-a))\cos(\omega_n(\tau-a))\isd\tau
		+ \int_{-\infty}^a\!\!\!\! w^2(t-\tau) e^{-2\lambda(a-\tau)}\isd\tau \\
	&=
		\frac{1}{2} \int_a^{t} w^2(t-\tau) \cos(\omega_{m-n}(\tau-a))\isd\tau
		+ \frac{1}{2} \int_a^{t} w^2(t-\tau) \cos(\omega_{m+n}(\tau-a))\isd\tau \nonumber \\
	&\qquad
	    + \int_{-\infty}^a w^2(t-\tau) e^{-2\lambda(a-\tau)}\isd\tau \\
	&= \frac{1}{2} \left(  I_{0,(n-m):\cos}(t) +  I_{0,(n+m):\cos}(t) \right)
\end{align}
where we use that $\omega_m \pm \omega_n = \frac{2\pi}{b-a}(m\pm n) = \omega_{m\pm n}$.
In the following, for $m>M$, recall that we adjust the frequency according to the construction of the variational Fourier features (see \eqref{eq:rgpcm-vffs}):
\begin{equation}
    k_{z_m}(\tau) = \beta_m(t) =  \sin(\omega_{m-M}(\tau-a)).
\end{equation}
For $m=0$ and $n>M$, we have sines:
\begin{align}
	I_{0,n:\sin}(t)
	&= \int_a^{t} w^2(t-\tau)  \sin(\omega_{n-M}(\tau-a))\isd\tau \\
	& \; = \int_0^{\omega_{n-M}(t-a)}  w^2(t-a-\tau/\omega_{n-M})  \sin(\tau) \frac{1}{\omega_{n-M}}\isd\tau \\
	&\;=  \frac{\alphat^2}{\omega_{n-M}}   \int_0^{\omega_{n-M}(t-a)} e^{-2\frac{\alpha}{\omega_{n-M}}(\omega_{n-M}(t-a)-\tau)}\sin(\tau)\isd\tau \\
	&= \frac{\alphat^2}{4\alpha^2+\omega_{n-M}^2}\left[
		\omega_{n-M} (e^{-2\alpha(t-a)} - \cos(\omega_{n-M}(t-a)))
		+ 2\alpha \sin(\omega_{n-M}(t-a))
	\right].
\end{align}
Next, for $m,n>M$,
\begin{align}
	I_{m:\sin,n:\sin}(t)
	&= \int_a^{t} w^2(t-\tau) \sin(\omega_{m-M}(\tau-a))\sin(\omega_{n-M}(\tau-a))\isd\tau \\
	&=
		\frac{1}{2} \int_a^{t} w^2(t-\tau) \cos(\omega_{m-n}(\tau-a))\isd\tau
		- \frac{1}{2} \int_a^{t} w^2(t-\tau) \cos(\omega_{m+n-2M}(\tau-a))\isd\tau  \\
	&= \vphantom{\int_a^{t}} \frac{1}{2}\left( I_{0,(n-m):\cos} - I_{0,(n+m-2M):\cos} \right).
\end{align}
Finally, for $0<m\leq M$ and $n>M$, we have both cosines and sines:
\begin{align}
    I_{m:\cos,n:\sin}(t)
	&= \int_a^{t} w^2(t-\tau) \cos(\omega_m(\tau-a))\sin(\omega_{n-M}(\tau-a))\isd\tau \\
	&=
		\frac{1}{2} \int_a^{t} w^2(t-\tau) \sin(\omega_{m+(n-M)}(\tau-a))\isd\tau
	    + \frac{1}{2} \int_a^{t} w^2(t-\tau) \sin(\omega_{(n-M)-m}(\tau-a))\isd\tau \\
	&= \vphantom{\int_a^{t}}  \frac{1}{2}\left(  I_{0,(n+m):\sin} +  I_{0,(n-m):\sin} \right).
\end{align}

\subsubsection{Integral
    \texorpdfstring{$\mI_{\vu x}(t)$}{Iux}
}\label{apdx::hux}
Denote the $(m, n)$\textsuperscript{th} element of $\mI_{\vu x}(t, t')$ by
\begin{align}
    [\mI_{\vu x}(t, t')]_{m,n}
    &= \int^{t}_{-\infty} \int^{t'}_{-\infty}
 		w(t-\tau)w(t'-\tau')  k_{u_m}(t -\tau)  k_{u_n}(t'-\tau') k_x(\tau,\tau') \isd\tau' \isd\tau   \\
 	&= \int_{0}^{\infty} \int_{0}^{\infty}
 		w(\tau)w(\tau')  k_{u_m}(\tau)k_{u_n}(\tau') k_x(t-\tau,t'-\tau') \isd\tau' \isd\tau  \\
 	&\eqqcolon \vphantom{\int_{0}^{\infty}} I_{m,n}(t, t').
\end{align}
Simplify
\begin{align}
I_{m,n}(t,t')  & =  	\int_{0}^{t_{u,n}} \int_{0}^{t_{u,m}}
 	\alphat^2 \gammat^2e^{-\alpha (\tau +\tau') - \gamma(t_{u,m} -\tau) - \gamma(t_{u,n} - \tau') - \lambda|(\tau - \tau') - (t - t')|}
 	\isd\tau \isd\tau' \\
 	&\quad =
 	\alphat^2 \gammat^2 e^{-\gamma(t_{u,m} + t_{u,n})}
 	\int_{0}^{t_{u,n}} \int_{0}^{t_{u,m}}
 	 e^{(\gamma -\alpha) (\tau +\tau') - \lambda|(\tau - \tau') - (t - t')|}
 	\isd\tau \isd\tau'.
\end{align}
Note that $I_{m,n}(t,t)$ is invariant of $t$.
The integral $I_{m,n}(t,t)$ can be computed with the following propositions.
\begin{proposition}
	Suppose that $a\ge 0$ and $b \ge 0$.
	Then
	\begin{equation}
		\int_0^a \int_0^b e^{c(\tau + \tau') - d|\tau - \tau'|}\isd{\tau}\isd{\tau'}
		=\frac{1}{c^2 - d^2}\left(
			1 + \frac{d}{c}\left(1 - e^{2c(a \land b)}\right)
			- e^{c a - d |a|}
			- e^{c b - d |b|}
			+ e^{c(a + b) - d|a - b|}
		\right).
	\end{equation}
\end{proposition}
\begin{proof}
	Suppose that $b \ge a$.
	Then $a - b = -|a - b|$ and $a = a \land b$.
	We simply calculate:
	\begin{align}
		\int_0^a \int_0^b e^{c(\tau + \tau') - d|\tau - \tau'|}\isd{\tau}\isd{\tau'}
		&=
			\int_0^a \int_0^{\tau'} e^{c(\tau + \tau') - d(\tau' - \tau)}\isd{\tau}\isd{\tau'}
			+ \int_0^a \int_{\tau'}^{b} e^{c(\tau + \tau') - d(\tau - \tau')}\isd{\tau}\isd{\tau'} \\
		&=
			\int_0^a e^{(c - d) \tau'} \int_0^{\tau'} e^{(c + d)\tau}\isd{\tau}\isd{\tau'}
			+ \int_0^a e^{(c + d) \tau'} \int_{\tau'}^{b} e^{(c - d) \tau}\isd{\tau}\isd{\tau'} \\
		&=
			\frac{1}{c + d}\int_0^a e^{(c - d) \tau'} \left(
				e^{(c + d)\tau'} - 1
			\right)\isd{\tau'} \nonumber \\
		&\phantom{=} \qquad
			+ \frac{1}{c - d}\int_0^a e^{(c + d) \tau'} \left(
				e^{(c - d)b} - e^{(c - d)\tau'}
			\right)\isd{\tau'} \\
		&=
			\frac{1}{c + d}\int_0^a\left(
				e^{2c \tau'}
				- e^{(c - d)\tau'}
			\right)\isd{\tau'} \nonumber \\
		&\phantom{=} \qquad
			+ \frac{1}{c - d}\int_0^a \left(
				e^{(c - d)b}e^{(c + d) \tau'}
				- e^{2c \tau'}
			\right)\isd{\tau'} \\
		&=
			\left[
				\frac{1}{2c (c + d)}\left(
					e^{2ca} - 1
				\right)
				-
				\frac{1}{c^2 - d^2}\left(
					e^{(c - d)a} - 1
				\right)
			\right] \nonumber \\
		&\phantom{=} \qquad + \left[
				\frac{e^{(c - d)b}}{c^2 - d^2}\left(
					e^{(c + d)a} - 1
				\right)
				-
				\frac{1}{2c (c - d)}\left(
					e^{2ca} - 1
				\right)
			\right] \\
		&=
			\frac{1}{2c}
			\left(
				\frac{1}{c + d} - \frac{1}{c - d}
			\right)
			\left(
					e^{2ca} - 1
			\right) \nonumber \\
		&\phantom{=} \qquad + \frac{1}{c^2 - d^2}\left(
			1
			- e^{(c - d)a}
			- e^{(c - d)b}
			+ e^{(c + d)a + (c - d)b}
		\right) \\
		&=
			\frac{d}{c}\frac{1}{c^2 - d^2}
			\left(
					1 - e^{2ca}
			\right)
		+ \frac{1}{c^2 - d^2}\left(
			1
			- e^{c a - d a}
			- e^{c b - d b}
			+ e^{c(a + b) + d(a - b)}
		\right).
	\end{align}
\end{proof}
\begin{proposition}
	Suppose that $ab \le 0$.
	Then
	\begin{equation}
		\int_0^a \int_0^b e^{c(\tau + \tau') - d|\tau - \tau'|}\isd{\tau}\isd{\tau'}
		=\frac{1}{c^2 - d^2}\left(
			1
			- e^{c a - |d| a}
			- e^{c b - |d| b}
			+ e^{c(a + b) - d|a - b|}
		\right).
	\end{equation}
\end{proposition}
 We can finally use the symmetry in $c$ to get the result for all $a,b\in \R$:
\newcommand{\sign}{\operatorname{sign}}
\begin{proposition}
	For all $a,b\in \R$,
	\begin{align}
		&\int_0^a \int_0^b e^{c(\tau + \tau') - d|\tau - \tau'|}\isd{\tau}\isd{\tau'} \nonumber \\
		&\qquad=\frac{1}{c^2 - d^2}\bigg(
			\ind(ab \ge 0)\frac{d \sign(a)}{c}\left(1 - e^{2c\sign(a)(|a| \land |b|)}\right)
			+ 1
			- e^{c a - d |a|}
			- e^{c b - d |b|}
			+ e^{c(a + b) - d|a - b|}
		\bigg).
	\end{align}
\end{proposition}
Putting everything together, we have the result
\begin{align*}
	I_{m,n}(t,t)
	&=\frac{\alphat^2 \gammat^2 e^{-\gamma(t_{u,m} + t_{u,n})}}{(\gamma - \alpha)^2 - \lambda^2}
	\bigg(
			\frac{\lambda}{\gamma - \alpha} \left(
				1 - e^{2(\gamma - \alpha)(t_{u,m} \land t_{u,n})}
			\right)
			\\
	&\hphantom{\;=\frac{\alphat^2 \gammat^2 e^{-\gamma(t_{u,m} + t_{u,n})}}{(\gamma - \alpha)^2 - \lambda^2}\bigg(}
			+ 1
			- e^{(\gamma - \alpha - \lambda) t_{u,m}}
			- e^{(\gamma - \alpha - \lambda) t_{u,n}}
			+ e^{(\gamma - \alpha)(t_{u,m} + t_{u,n}) - \lambda|t_{u,m} - t_{u,n}|}
	\bigg).
\end{align*}

\subsubsection{Integral \texorpdfstring{$\mI_{\vu\vz}(t)$}{Iuz}}\label{apdx::iuz}
Denote the $(m,k)$\textsuperscript{th} element of $\mI_{\vu\vz}(t)$ by
\begin{equation}
    [\mI_{\vu\vz}(t)]_{m,k}
    = \int_{0}^{\infty}
			w(\tau)
			k_{u_m}(\tau)
			k_{z_k}(t- \tau)
		\isd{\tau}
	\eqqcolon I_{m,k}(t).
\end{equation}
Simplify
\begin{align}
	I_{m,k}(t)
	&=
		\alphat \gammat
		\int_{0}^{t_{u,m}}
			e^{- \alpha \tau - \gamma (t_{u,m} - \tau)}
			k_{z_k}(t - \tau)
		\isd{\tau} \\
	&=
		\alphat \gammat e^{- \gamma t_{u,m}}
		\int_{0}^{t_{u,m}}
			e^{(\gamma - \alpha) \tau}
			k_{z_k}(t - \tau)
		\isd{\tau} \\
	&=
		\alphat \gammat e^{- \gamma t_{u,m} + (\gamma - \alpha) t}
		\int^{t }_{t - t_{u,m}}
			e^{(-\gamma + \alpha)\tau}
			k_{z_k}(\tau)
		\isd{\tau}.
\end{align}
We analyse the result case by case.
Denote
\begin{equation}
	I(l, u, k)
	= \int_{l}^{u}
		e^{(-\gamma + \alpha)\tau}
		k_{z_k}(\tau)
	\isd{\tau}.
\end{equation}
\paragraph{The case $k=0$, $a \le l \le b$, and $a \le u \le b$:}
\begin{equation}
	I(l, u, 0)
	= \int_l^u
		e^{(-\gamma + \alpha)\tau}
		\isd \tau
	= \frac{1}{-\gamma + \alpha}\left(
		e^{(-\gamma + \alpha)u}
		- e^{(-\gamma + \alpha)l}
	\right).
\end{equation}
\paragraph{The case $0 \le k \le K$ and $l, u \le a$:}\footnote{\label{fn1} We could define $\vt_u$ and  $[a,b]$ such that this case never happens.}
\begin{equation}
	I(l, u, k)
	= \int_l^u
		e^{(-\gamma + \alpha)\tau - \lambda(a - \tau)}
		\isd \tau
	= \frac{e^{- \lambda a}}{-\gamma + \alpha + \lambda}\left(
		e^{(-\gamma + \alpha + \lambda)u}
		- e^{(-\gamma + \alpha + \lambda)l}
	\right).
\end{equation}
\paragraph{The case $0 \le k \le K$ and $l, u \ge b$:}\textsuperscript{\ref{fn1}}
\begin{equation}
	I(l, u, k)
	= \int_l^u
		e^{(-\gamma + \alpha)\tau - \lambda(\tau - b)}
		\isd \tau
	= \frac{e^{\lambda b}}{-\gamma + \alpha - \lambda}\left(
		e^{(-\gamma + \alpha - \lambda)u}
		- e^{(-\gamma + \alpha - \lambda)l}
	\right).
\end{equation}
\paragraph{The case $M < k \le 2 K$ and either $l, u \le a$ or $l, u \ge b$:}\textsuperscript{\ref{fn1}}
\begin{equation}
	I(l, u, k) = 0.
\end{equation}
\paragraph{The case $1 \le k \le M$, $a \le l \le b$, and $a \le u \le b$:}
\begin{align}
	I(l, u, k)
	+ i I(l, u, k + M) 
	&= \int_l^u
		e^{(-\gamma + \alpha)\tau + i \omega_k (\tau - a)}
		\isd \tau \\
	&=
		\frac{e^{- i \omega_k a}}{-\gamma + \alpha + i \omega_k}
		\left(
			e^{(-\gamma + \alpha + i \omega_k)u}
			-  e^{(-\gamma + \alpha + i \omega_k)l}
		\right) \\
	&=
		\frac{-\gamma + \alpha - i \omega_k}{(-\gamma + \alpha)^2 + \omega_k^2}
		\left(
			e^{(-\gamma + \alpha)u + i \omega_k (u - a)}
			-  e^{(-\gamma + \alpha)l + i \omega_k (l - a)}
		\right),
\end{align}
which shows that
\begin{align}
	((-\gamma + \alpha)^2 + \omega_k^2)I(l, u, k)
	&=
		e^{(-\gamma + \alpha) u}
		[
			(- \gamma + \alpha)\cos(\omega_k(u - a))
			+ \omega_k\sin(\omega_k(u - a))
		] \nonumber \\
	&\phantom{=}\qquad
		- e^{(-\gamma + \alpha) l}
		[
			(- \gamma + \alpha)\cos(\omega_k(l - a))
			+ \omega_k\sin(\omega_k(l - a))
		]
\end{align}
and
\begin{align}
	((-\gamma + \alpha)^2 + \omega_k^2)I(l, u, k + M)
	&=
		e^{(-\gamma + \alpha) u}
		[
			(\gamma - \alpha)\sin(\omega_k(u - a))
			+ \omega_k\cos(\omega_k(u - a))
		] \nonumber \\
	&\phantom{=}\qquad
		- e^{(-\gamma + \alpha) l}
		[
			(\gamma - \alpha)\sin(\omega_k(l - a))
			+ \omega_k\cos(\omega_k(l - a))
		].
\end{align}

\end{document}